\documentclass{article}

\usepackage{amsmath,amsfonts,bm}

\newcommand{\figleft}{{\em (Left)}}
\newcommand{\figcenter}{{\em (Center)}}
\newcommand{\figright}{{\em (Right)}}

\def\eqref#1{equation~\ref{#1}}

\def\1{\bm{1}}

\DeclareMathAlphabet{\mathsfit}{\encodingdefault}{\sfdefault}{m}{sl}
\SetMathAlphabet{\mathsfit}{bold}{\encodingdefault}{\sfdefault}{bx}{n}

\newcommand{\E}{\mathbb{E}}

\DeclareMathOperator*{\argmax}{arg\,max}

\usepackage{amsfonts}       %
\usepackage{amssymb}
\usepackage{amsthm}
\usepackage{bbm}
\usepackage{booktabs}       %
\usepackage{cancel}
\usepackage{color}
\usepackage{enumitem}
\usepackage[T1]{fontenc}    %
\usepackage{graphicx}
\usepackage[draft]{hyperref}
\usepackage[utf8]{inputenc} %
\usepackage{mathtools}
\usepackage{microtype}      %
\usepackage{multirow}
\usepackage{nicefrac}       %
\usepackage{soul}
\usepackage{subcaption}
\usepackage{todonotes}
\usepackage{url}            %
\usepackage{wrapfig}
\usepackage{xcolor}
\usepackage{pifont}
\usepackage{amsmath}
\newcommand{\cmark}{\text{\ding{51}}}
\newcommand{\xmark}{\text{\ding{55}}}

\usepackage[accepted]{icml2020}

\usepackage{graphbox}  %

\DeclarePairedDelimiterX{\infdivx}[2]{(}{)}{%
  #1\;\delimsize\|\;#2%
}
\newcommand{\kl}{D_{\text{KL}}\infdivx}
\newtheorem{theorem}{Theorem}[section]
\newtheorem{assumption}{Assumption}
\newtheorem{lemma}[theorem]{Lemma}
\newtheorem{proposition}[theorem]{Proposition}
\theoremstyle{definition}
\newtheorem{definition}{Definition}[section]

\icmltitlerunning{Efficient Exploration via State Marginal Matching}

\begin{document}

\twocolumn[
\icmltitle{Efficient Exploration via State Marginal Matching}

\icmlsetsymbol{equal}{*}

\begin{icmlauthorlist}
\icmlauthor{Lisa Lee}{equal,cmu,google}
\icmlauthor{Benjamin Eysenbach}{equal,cmu,google}
\icmlauthor{Emilio Parisotto}{cmu}
\icmlauthor{Eric Xing}{cmu}
\icmlauthor{Sergey Levine}{berkeley,google}
\icmlauthor{Ruslan Salakhutdinov}{cmu}
\end{icmlauthorlist}

\icmlaffiliation{cmu}{Carnegie Mellon University}
\icmlaffiliation{berkeley}{UC Berkeley}
\icmlaffiliation{google}{Google Brain}

\icmlcorrespondingauthor{Lisa Lee}{lslee@cs.cmu.edu}

\icmlkeywords{Deep Reinforcement Learning, Exploration, Density Estimation, Distribution Matching}

\vskip 0.3in
]

\printAffiliationsAndNotice{\icmlEqualContribution} %

\begin{abstract}
Exploration is critical to a reinforcement learning agent's performance in its given environment. Prior exploration methods are often based on using heuristic auxiliary predictions to guide policy behavior, lacking a mathematically-grounded objective with clear properties. 
In contrast, we recast exploration as a problem of \emph{State Marginal Matching} (SMM), where we aim to learn a policy for which the state marginal distribution matches a given target state distribution. The target distribution is a uniform distribution in most cases, but can incorporate prior knowledge if available.
In effect, SMM amortizes the cost of \emph{learning to explore} in a given environment. 
The SMM objective can be viewed as a two-player, zero-sum game between a state density model and a parametric policy, an idea that we use to build an algorithm for optimizing the SMM objective. Using this formalism, we further demonstrate that prior work approximately maximizes the SMM objective, offering an explanation for the success of these methods. 
On both simulated and real-world tasks, we demonstrate that agents that  directly optimize the SMM objective explore faster and adapt more quickly to new tasks as compared to prior exploration methods. \footnote{Videos and code: \url{https://sites.google.com/view/state-marginal-matching}}
\vspace{-1em}
\end{abstract}

\section{Introduction}
\vspace{-0.5em}

Reinforcement learning (RL) algorithms must be equipped with exploration mechanisms to effectively solve tasks with long horizons and limited or delayed reward signals. These tasks arise in many real-world applications where providing human supervision is expensive.

Exploration for RL has been studied in a wealth of prior work.
The optimal exploration strategy is intractable to compute in most settings, motivating work on tractable heuristics for exploration~\citep{kolter2009near}.
Exploration methods based on random actions have limited ability to cover a wide range of states. More sophisticated techniques, such as intrinsic motivation, accelerate learning in the single-task setting. However, these methods have two limitations: (1) First, they lack an explicit objective to quantify ``good exploration,'' but rather argue that exploration arises implicitly through some iterative procedure. Lacking a well-defined optimization objective, it remains unclear what these methods are doing and why they work. Similarly, the lack of a metric to quantify exploration, even if only for evaluation, makes it %
difficult
to compare exploration methods and assess progress in this area.
(2) The second limitation is that these methods target the single-task setting. Because these methods aim to converge to the optimal policy for a particular task, it is difficult to repurpose these methods to solve multiple tasks.

We address these shortcomings by recasting exploration as a problem of \emph{State Marginal Matching (SMM)}: Given a target state distribution, we learn a policy for which the state marginal distribution matches this target distribution. Not only does the SMM problem provide a clear and explicit objective for exploration, but it also provides a convenient mechanism to incorporate prior knowledge about the task through the target distribution --- whether in the form of safety constraints that the agent should obey; preferences for some states over other states; reward shaping; or the relative importance of each state dimension for a particular task. Without any prior information, the SMM objective reduces to maximizing the marginal state entropy $\mathcal{H}[s]$, %
which encourages the policy to visit all states.%

In this work, we study state marginal matching as a metric for task-agnostic exploration.
While this class of objectives has been considered in~\citet{hazan2018provably}, we build on this prior work in a number of dimensions:

\textbf{(1)} {We argue that the SMM objective is an effective way to learn a \emph{single, task-agnostic exploration policy} that can be used for solving many downstream tasks, amortizing the cost of learning to explore for each task.} Learning a single exploration policy is considerably more difficult than doing exploration throughout the course of learning a single task. The latter is done by intrinsic motivation~\citep{pathak2017curiosity,tang2017exploration,oudeyer2007intrinsic} and count-based exploration~\citep{bellemare2016unifying}, which can effectively explore to find states with high reward, at which point the agent can decrease exploration and increase exploitation of those high-reward states. While these methods perform efficient exploration for learning a single task, we show in Sec.~\ref{sec:prediction-error} that the policy at any particular iteration is not a good exploration policy. %

In contrast, maximizing $\mathcal{H}[s]$ produces a stochastic policy at convergence that visits states in proportion to their density under a target distribution. We use this policy as an exploration prior in our multi-task experiments, and also prove that this policy is optimal for a class of goal-reaching tasks (Appendix~\ref{appendix:hitting-time}).

\textbf{(2)} {We explain how to optimize the SMM objective properly.} By viewing the objective as a two-player, zero-sum game between a state density model and a parametric policy, we propose a practical algorithm to jointly learn the policy and the density by using fictitious play~\citep{brown1951iterative}.

We further decompose the SMM objective into a mixture of distributions, and derive an algorithm for learning a mixture of policies that resembles the mutual-information objectives in recent work~\citep{achiam2018variational,eysenbach2018diversity,co2018self}. Thus, these prior work may be interpreted as also almost doing distribution matching, with the caveat that they omit the state entropy term.

\textbf{(3)} {Our analysis provides a unifying view of prior exploration methods as \emph{almost} performing distribution matching}. We show that exploration methods based on predictive error approximately optimizes the same SMM objective, offering an explanation for the success of these methods. However, they omit a crucial historical averaging step, potentially explaining why they do not converge to an exploratory policy.

\textbf{(4)} {We demonstrate on complex RL tasks that optimizing the SMM objective allows for faster exploration and adaptation} than prior state-of-the-art exploration methods.

In short, our paper contributes a method to measure, amortize, and understand exploration.

\section{State Marginal Matching}\label{section:state-marginal-matching}

In this section, we start by showing that exploration methods based on prediction error do not acquire a single exploratory policy. This motivates us to define the State Marginal Matching problem as a principled objective for \emph{learning to explore}. We then introduce an extension of the SMM objective using a mixture of policies. %

\subsection{Why Prediction Error is Not Enough}
\label{sec:prediction-not-enough}
Exploration methods based on prediction error~\citep{burda2018exploration,stadie2015incentivizing, pathak2017curiosity, schmidhuber1991possibility,chentanez2005intrinsically} do not converge to an exploratory policy, even in the absence of extrinsic reward. For example, consider the asymptotic behavior of ICM~\citep{pathak2017curiosity} in a deterministic MDP, such as the Atari games where it was evaluated. At convergence, the predictive model will have zero error in all states, so the exploration bonus is zero -- the ICM objective has no effect on the policy at convergence. Similarly, consider the exploration bonus in Pseudocounts~\citep{bellemare2016unifying}: $1 / \hat{n}(s)$, where $\hat{n}(s)$ is the (estimated) number of times that state $s$ has been visited. In the infinite limit, each state has been visited infinitely many times, so the Pseudocount exploration bonus also goes to zero --- Pseudocounts has no effect at convergence. Similar reasoning can be applied to other methods based on prediction error~\citep{burda2018exploration, stadie2015incentivizing}.
More broadly, we can extend this analysis to stochastic MDPs, where we consider an abstract exploration algorithm that alternates between computing some intrinsic reward and performing RL (to convergence) on that intrinsic reward. Existing prediction-error exploration methods are all special cases. At each iteration, the RL step solves a fully-observed MDP, which always admits a deterministic policy as a solution~\citep{puterman2014markov}.
Thus, any exploration algorithm in this class cannot converge to a single, exploratory policy.
Next, we present an objective which, when optimized, yields a single exploratory policy.

\vspace{-0.5em}
\subsection{The State Marginal Matching Objective}
\label{sec:smm-objective}
We consider a parametric policy $\pi_\theta \in \Pi \triangleq \{\pi_\theta \mid \theta \in \Theta\}$, e.g. a policy parameterized by a deep network, that chooses actions $a \in \mathcal{A}$ in a Markov Decision Process (MDP) with fixed episode lengths $T$,
dynamics distribution $p(s_{t+1} \mid s_t, a_t)$,
and initial state distribution $p_0(s)$. The MDP together with the policy $\pi_\theta$ form an implicit generative model over states.
We define the \emph{state marginal distribution} $\rho_{\pi}(s)$ as the probability that the policy visits state $s$:
\begin{equation*}
    \rho_{\pi}(s) \triangleq \E_{\substack{
    s_1 \sim p_0(S),\\
    a_t \sim \pi_\theta(A \mid s_t)\\
    s_{t+1} \sim p(S \mid s_t, a_t)
    }} \left[\frac{1}{T}\sum_{t=1}^T \mathbbm{1}(s_t = s) \right]
\end{equation*}
The state marginal distribution $\rho_{\pi}(s)$ is a distribution of states, not trajectories: it is the distribution over states visited in a finite-length episode, not the stationary distribution of the policy after infinitely many steps.\footnote{$\rho_{\pi}(s)$ approaches the policy's stationary distribution in the limit as the episodic horizon $T \rightarrow \infty$.}

\begin{figure}[t]
    \centering
    \includegraphics[width=\columnwidth]{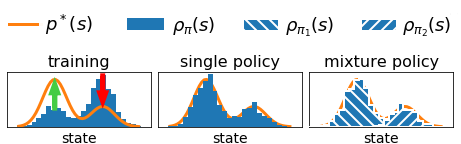}
    \vspace{-2.0em}
    \caption{\textbf{State Marginal Matching}: \figleft \;Our goal is to learn a policy whose state distribution $\rho_\pi(s)$ matches some target density $p^*(s)$. Our algorithm iteratively increases the reward on states visited too infrequently (green arrow) and decreases the reward on states visited too frequently (red arrow). \figcenter\; At convergence, these two distributions are equal. \figright\; For complex target distributions, we use a mixture of policies $\rho_{\pi}(s) = \int \rho_{\pi_z}(s) p(z) dz$.}
    \vspace{-0.5em}
    \label{fig:teaser}
\end{figure}

We assume that we are given a target distribution $p^*(s)$ over states $s \in \mathcal{S}$ that encodes our belief about the tasks we may be given at test-time. For example, a roboticist might assign small values of $p^*(s)$ to states that are dangerous, regardless of the desired task. Alternatively, we might also learn $p^*(s)$ from data about human preferences~\citep{christiano2017deep}. For goal-reaching tasks, we can analytically derive the optimal target distribution (Appendix~\ref{appendix:hitting-time}). Given $p^*(s)$, our goal is to find a parametric policy that is ``closest'' to this target distribution, where we measure discrepancy using the Kullback-Leibler (KL) divergence: %
\vspace{-0.5em}
\begin{align}
    &\min_{\pi \in \Pi} \kl{\rho_{\pi}(s)}{p^*(s)} \nonumber \\
    \triangleq &\max_{\pi \in \Pi} \E_{\rho_{\pi}(s)}\; \log p^*(s) + \mathcal{H}_\pi[s] \label{eq:decomp-1}
    \vspace{-5em}
\end{align}
The SMM objective in Eq.~\ref{eq:decomp-1} can be viewed as maximizing the pseudo-reward
\mbox{$r(s) \triangleq \log p^*(s) - \log \rho_{\pi}(s)$}, which assigns positive utility to states that the agent visits too infrequently and negative utility to states visited too frequently (see Fig.~\ref{fig:teaser}). Maximizing this pseudo-reward is not a RL problem because the pseudo-reward depends on the policy. Maximizing the first term alone without the state entropy regularization term will converge to the \emph{mode} of the target distribution, rather than do distribution matching. Moreover, the SMM objective regularizes the entropy of the state distribution, not the conditional distribution of actions given states, as done in MaxEnt RL~\citep{ziebart2008maximum,haarnoja2018soft}. This results in exploration in the space of states rather than in actions.

\subsection{Better SMM with Mixtures of Policies}
\label{sec:mixture-policies}

Given the challenging problem of exploration in large state spaces, it is natural to wonder whether we can accelerate exploration by automatically decomposing the potentially-multimodal target distribution into a mixture of ``easier-to-learn'' distributions and learn a corresponding set of policies to do distribution matching for each component. Note that the mixture model we introduce here is orthogonal to the historical averaging step discussed before.
Using $\rho_{\pi_z}(s)$ to denote the state distribution of the policy conditioned on the latent variable $z \in \mathcal{Z}$, the state marginal distribution of the mixture of policies $\pi_z$ with prior $p(z)$ is
\begin{equation}
    \rho_{\pi}(s) = \int_\mathcal{Z} \rho_{\pi_z}(s) p(z) dz
    = \E_{z \sim p(z)}\left[ \rho_{\pi_z}(s) \right].
    \vspace{-0.5em}
\end{equation}
As before, we will minimize the KL divergence between this mixture distribution and the target distribution.
Using Bayes' rule to re-write $\rho_{\pi}(s)$ in terms of conditional probabilities, we obtain the following optimization problem:
\begin{align}
&\max_{\substack{\pi_z, \\z \in \mathcal{Z}}} \E_{\substack{p(z),\\\rho_{\pi_z}(s)}} \left[ r_z(s) \right]\;, \label{eq:smm-mop-objective}\\
&r_z(s) \triangleq \underbrace{\log p^*(s)}_{(a)} - \underbrace{\log \rho_{\pi_z}(s)}_{(b)} + \underbrace{\log p(z \mid s)}_{(c)}  - \underbrace{\log p(z)}_{(d)}\nonumber
\vspace{-1cm}
\end{align}
Intuitively, this says that the agent should go to states (a) with high density under the target state distribution, (b) where this agent has not been before, and (c) where this agent is clearly distinguishable from the other agents. The last term (d) says to explore in the space of mixture components $z$.
This decomposition resembles the mutual-information objectives in recent work~\citep{achiam2018variational,eysenbach2018diversity,co2018self}. Thus, one interpretation of our work is as explaining that mutual information objectives almost perform distribution matching. The caveat is that prior work omits the state entropy term $-\log \rho_{\pi_z}(s)$ which provides high reward for visiting novel states, possibly explaining why these previous works have failed to scale to complex tasks.

In Appendix~\ref{appendix:rl-goals}, we also discuss how goal-conditioned RL~\citep{kaelbling1993learning,schaul2015universal} can be viewed as a special case of State Marginal Matching when the goal-sampling distribution is learned jointly with the policy.

\vspace{-0.5em}
\section{A Practical Algorithm}
\label{sec:game-theory}
In this section, we develop a principled algorithm for maximizing the state marginal matching objective. We then propose an extension of this algorithm based on mixture modelling, an extension with close ties to prior work.

\begin{figure}
    \centering
\vspace{-0.5em}
 \begin{algorithm}[H]
    \caption{Learning to Explore via Fictitious Play \label{alg:smm} }
  \begin{algorithmic}[0]
     \footnotesize
        \STATE \textbf{Input:} Target distribution $p^*(s)$
        \STATE Initialize policy $\pi(a \mid s)$, density model $q(s)$, replay buffer $\mathcal{B}$.
\WHILE{not converged}
\STATE $q^{(m)} \gets \argmax_q \mathbb{E}_{s \sim \mathcal{B}^{(m-1)}} \left[ \log q(s) \right]$
\STATE $\pi^{(m)} \gets \argmax_\pi \E_{s \sim \rho_{\pi}(s)} \left[ r(s) \right]$ where $r(s) \triangleq \log p^*(s) -\log q^{(m)}(s)$
\STATE $\mathcal{B}^{(m)} \gets \mathcal{B}^{(m-1)} \cup \{ (s_t,a_t,s_{t+1}) \}_{t=1}^T$ with new transitions sampled from $\pi^{(m)}$
\ENDWHILE
\STATE \textbf{return} historical policies $\{\pi^{(1)}, \cdots, \pi^{(m)}\}$
\end{algorithmic} 
  \end{algorithm}
\vspace{-1.6em}
\caption*{\footnotesize An algorithm for optimizing the State Marginal Matching objective (Eq.~\ref{eq:decomp-1}). The algorithm iterates between (1) fitting a density model $q^{(m)}$ and (2) training the policy $\pi^{(m)}$ with a RL objective  to optimize the expected return w.r.t.\ the updated reward function $r(s)$. The algorithm returns the collection of policies from each iteration, which do distribution matching in aggregate.
}
\vspace{-1.5em}
\end{figure} 

\vspace{-0.5em}
\subsection{Optimizing the State Marginal Matching Objective}
Optimizing Eq.~\ref{eq:decomp-1}
is more challenging than standard RL because the reward function itself depends on the policy.
To break this cyclic dependency, we introduce a parametric state density model $q_\psi(s) \in Q \triangleq \{ q_\psi \mid \psi \in \Psi \}$ to approximate the policy's state marginal distribution, $\rho_{\pi}(s)$. We assume that the class of density models $Q$ is sufficiently expressive to represent every policy:
\begin{assumption}\label{assumption-existence}
For every policy $\pi \in \Pi$, there exists $q \in Q$ such that $\kl{\rho_\pi(s)}{q(s)} = 0$. \label{ass:density}
\end{assumption}
Under this assumption, optimizing the policy w.r.t.\ this approximate distribution $q(s)$ will yield the same solution as Eq.~\ref{eq:decomp-1} (see Appendix~\ref{sec:proofs} for the proof):
\begin{proposition}\label{lemma:max-min-equivalence}
Let policies $\Pi$ and density models $Q$ satisfying Assumption~\ref{ass:density} be given. For any target distribution $p^*$, the following optimization problems are equivalent:
\begin{align}
    \max_\pi &\E_{\rho_{\pi}(s)}[\log p^*(s) - \log \rho_{\pi}(s)] \nonumber\\= \max_\pi \min_q &\E_{\rho_{\pi}(s)}[\log p^*(s) - \log q(s)] \label{eq:min-max-obj}
\end{align}
\end{proposition}
\vspace{-1.5em}
Solving the new max-min optimization problem is equivalent to finding the Nash equilibrium of a two-player, zero-sum game: a \emph{policy player} chooses the policy $\pi$ while the \emph{density player} chooses the density model $q$.
To avoid confusion, we use \emph{actions} to refer to controls $a \in \mathcal{A}$ output by the policy $\pi$ in the traditional RL problem and \emph{strategies} to refer to the decisions of the policy player  $\pi \in \Pi$ and density player $q \in Q$. The Nash existence theorem~\citep{nash1951non} proves that such a stationary point always exists for such a two-player, zero-sum game.

One common approach to saddle point games is to alternate between updating player A w.r.t.\ player B, and updating player B w.r.t.\ player A. However, games such as Rock-Paper-Scissors illustrate that such a greedy approach is not guaranteed to converge to a stationary point.
A slight variant, \emph{fictitious play}~\citep{brown1951iterative} does converge to a Nash equilibrium in finite time~\citep{robinson1951iterative,daskalakis2014counter}. At each iteration, each player chooses their best strategy in response to the \emph{historical average} of the opponent's strategies.
In our setting, fictitious play alternates between fitting the density model to the historical average of policies $\bar{\rho}_m(s) \triangleq \frac{1}{m}\sum_{i=1}^m \rho_{\pi_i}(s)$ (Eq.~\ref{eq:density-1}), and updating the policy with RL to minimize the log-density of the state, using a historical average of the density models $\bar{q}_{m}(s) \triangleq \frac{1}{m} \sum_{i=1}^{m} q_i(s)$ (Eq.~\ref{eq:density-2}):
\begin{align}
    q_{m+1} &\gets \argmax_q \E_{s \sim \bar{\rho}_m(s)}[\log q(s)] \label{eq:density-1} \\
    \pi_{m+1} &\gets \argmax_\pi \E_{s \sim \rho_{\pi}(s)} \left[\log p^*(s) -\log \bar{q}_{m}(s)\right] \label{eq:density-2}
\end{align}
Crucially, the exploration policy is not the last policy, $\pi_{m+1}$, but rather the historical average policy:
\begin{definition}
A \emph{historical average policy} $\bar{\pi}(a \mid s)$, parametrized by a collection of policies $\pi_1, \cdots, \pi_m$, is a policy that randomly samples one of the policy iterates \mbox{$\pi_i \sim \text{Unif}[\pi_1, \cdots, \pi_m]$} at the start of each episode and takes actions according to that policy in the episode. %
\end{definition}
We summarize the resulting algorithm in Alg.~\ref{alg:smm}.
In practice, we can efficiently implement Eq.~\ref{eq:density-1} and avoid storing the policy parameters from every iteration by instead storing sampled states from each iteration.
Alg.~\ref{alg:smm} looks similar to prior exploration methods based on prediction-error, suggesting that we might use SMM to understand how these prior methods work (Sec~\ref{sec:prediction-error}).

\vspace{-0.5em}
\subsection{Extension to Mixtures of Policies}

We refer to the algorithm with mixture modelling as SM4 (State Marginal Matching with Mixtures of Mixtures), and summarize the method in Alg.~\ref{alg:smm-mop} in the Appendix. The algorithm (1) fits a density model $q^{(m)}_z(s)$ to approximate the state marginal distribution for each policy $\pi_z$; (2) learns a discriminator $d^{(m)}(z \mid s)$ to predict which policy $\pi_z$ will visit state $s$; and (3) uses RL to update each policy $\pi_z$ to maximize the expected return of its corresponding reward function $r_z(s)$ %
derived in Eq.~\ref{eq:smm-mop-objective}.

The only difference from Alg.~\ref{alg:smm} is that we learn a discriminator $d(z \mid s)$, in addition to updating the density models $q_z(s)$ and the policies $\pi_z(a \mid s)$. Jensen's inequality tells us that maximizing the log-density of the learned discriminator will maximize a lower bound on the true density (see~\citet{agakov2004algorithm}):
\vspace{-0.5em}
\begin{equation*}
    \E_{\substack{s \sim \rho_{\pi_z}(s), \\z \sim p(z)}}[\log d(z \mid s)] \le \E_{s \sim \rho_{\pi_z}(s), z \sim p(z)} [\log p(z \mid s) ]
\vspace{-0.5em}
\end{equation*}
The algorithm returns the historical average of mixtures of policies (a total of $n\cdot m$ policies). Note that updates for each $z$ can be conducted in parallel.

\section{Prediction-Error Exploration is Approximate State Marginal Matching}
\label{sec:prediction-error}

This section compares and contrasts SMM with prior exploration methods that use predictive-error, showing that these methods approximately optimize the same SMM objective when averaged over time, but otherwise exhibits oscillatory learning dynamics. Both the objectives and the optimization procedures are similar, but contain important yet subtle differences.

\begin{table}[t]
\centering
\begin{footnotesize}
\begin{tabular}{c|c|c|c}
     & inputs & targets & prior\\ \hline
    SMM (Ours) & $s$ & $s$ & $\cmark$ \\
    RND & $s$ & $e(s)$ & \xmark \\
    Forward Models%
    & $s, a$ & $s'$  & \xmark \\
    Inverse Models & $s, s'$ & $a$ & \xmark

\end{tabular}
\end{footnotesize}
\vspace{-0.5em}
\caption{\textbf{Exploration based on predictive-error}: A number of exploration methods operate by learning a function that predicts some target quantity given some input quantities, and using this function's error as an exploration bonus. Previous methods have omitted the prior term, which our method implicitly incorporates via historical averaging.
\label{table:predictive-error-exploration}}
\vspace{-1.5em}
\end{table}

\vspace{-0.5em}
\paragraph{Objectives} As introduced in Proposition~\ref{lemma:max-min-equivalence}, the state marginal matching objective can be viewed as a min-max objective (Eq.~\ref{eq:min-max-obj}). When the density model is a VAE and the target distribution $p^*(s)$ is uniform, this min-max objective looks like the prediction error between a state and itself, plus a regularizer:
\vspace{-0.5em}
\begin{equation*}
        \max_\pi \min_\psi \E_{\rho_{\pi}(s)} \left[\|f_\psi(s_t) - s_t\|_2^2 \right] + R_\pi(\psi), \label{eq:smm-min-max}
\vspace{-1.em}
\end{equation*}
where $f_\phi$ is our autoencoder and $R_\pi(\psi)$ is the KL penalty on the VAE encoder for the data distribution $\rho_\pi(s)$. Prior exploration methods look quite similar. For example, 
Exploration methods based on predictive error also optimize a min-max objective. For example, the objective for RND~\citep{burda2018exploration} is
\vspace{-0.5em}
\begin{equation*}
    \max_\pi \min_\psi \E_{\rho_{\pi}(s)} \left[\|f_\psi(s_t) - e(s_t)\|_2^2 \right],
    \vspace{-0.8em}
\end{equation*}
where $e(\cdot)$ is an encoder obtained by a randomly initialized neural network.
Exploration bonuses based on the predictive error of forward models~\citep{schmidhuber1991possibility,chentanez2005intrinsically,stadie2015incentivizing} have a similar form, but instead consider full transitions:
\vspace{-0.5em}
\begin{equation*}
    \max_\pi \min_\psi \E_{\rho_{\pi}(s)} \left[ \|f_\psi(s_t, a_t) - s_{t+1}\|_2^2 \right].
    \vspace{-0.8em}
\end{equation*}
Exploration bonuses derived from inverse models~\citep{pathak2017curiosity} look similar:
\vspace{-0.5em}
\begin{equation*}
    \max_\pi \min_\psi \E_{\rho_{\pi}(s)} \left\|f_\psi(s_t, s_{t+1}) - a_t\|_2^2 \right].
    \vspace{-.5em}
\end{equation*}
We summarize these methods in Table~\ref{table:predictive-error-exploration}. We believe that the prior term $R(\psi)$ in the SMM objective (Eq.~\ref{eq:smm-min-max}) that is omitted from the other objectives possibly explains why SMM continues to explore at convergence.

\vspace{-0.5em}
\paragraph{Optimization} Both SMM and prior exploration methods employ alternating optimization to solve their respective min-max problems. Prior work uses a greedy procedure that optimizes the policy w.r.t. the \emph{current} auxiliary model, and optimizes the auxiliary model w.r.t. the \emph{current} policy.  This greedy procedure often fails to converge, as we demonstrate experimentally in Section~\ref{sec:didactic-experiments}. In contrast, SMM uses fictitious play, a slight modification that optimizes the policy w.r.t. the \emph{historical average} of the auxiliary models and optimizes the auxiliary model w.r.t. the \emph{historical average} of the policies. Unlike the greedy approach, fictitious play is guaranteed to converge. This difference may explain why SMM learns better exploratory policies than prior methods. %

While prior works use a procedure that is not guaranteed to converge, they nonetheless excel at solving hard exploration tasks. We draw an analogy to fictitious play to explain their success. While these methods never acquire an exploratory policy, over the course of training they will eventually visit all states. In other words, the \emph{historical average} over policies will visit a wide range of states. Since the replay buffer exactly corresponds to this historical average over states, these methods will obtain a replay buffer with a diverse range of experience, possibly explaining why they succeed at solving hard exploration tasks. Moreover, this analysis suggests a surprisingly simple method for obtaining an exploration from these prior methods: use a mixture of the policy iterates throughout training. The following section will not only compare SMM against prior exploration methods, but also show that this historical averaging trick can be used to improve existing exploration methods.

\vspace{-0.5em}
\subsection{Didactic Experiments}
\label{sec:didactic-experiments}

\begin{figure}[t]
\centering
\vspace{-0.7em}
    \begin{subfigure}[b]{0.5\linewidth}
        \centering
        \includegraphics[width=\linewidth,trim=120pt 40pt 40pt 0,clip]{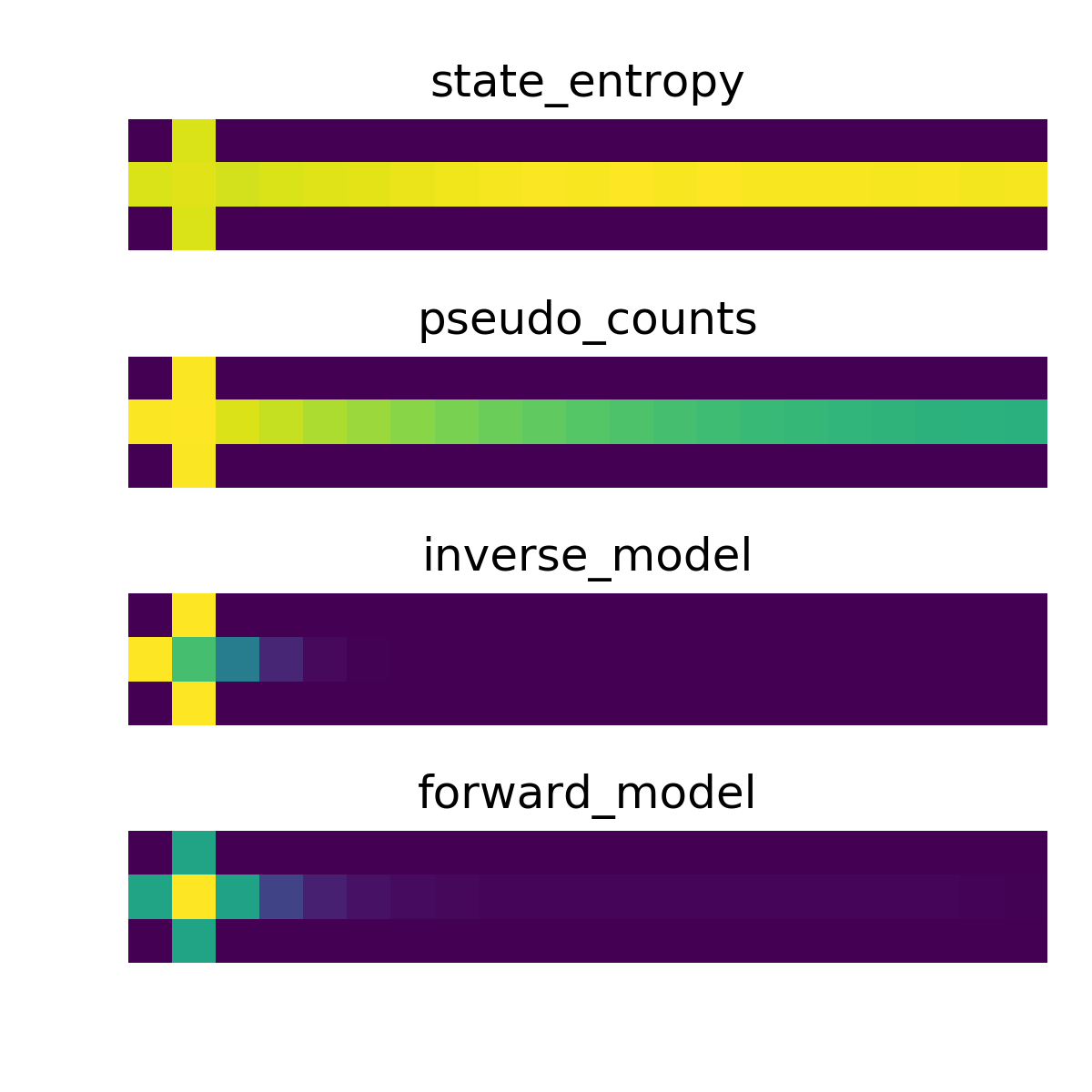}
        \vspace{-2.5em}
        \caption{State Marginals\label{fig:marginals}}
    \end{subfigure}%
    \hfill
    \begin{subfigure}[b]{0.5\linewidth}
        \centering
        \includegraphics[width=\linewidth,trim=120pt 40pt 40pt 0,clip]{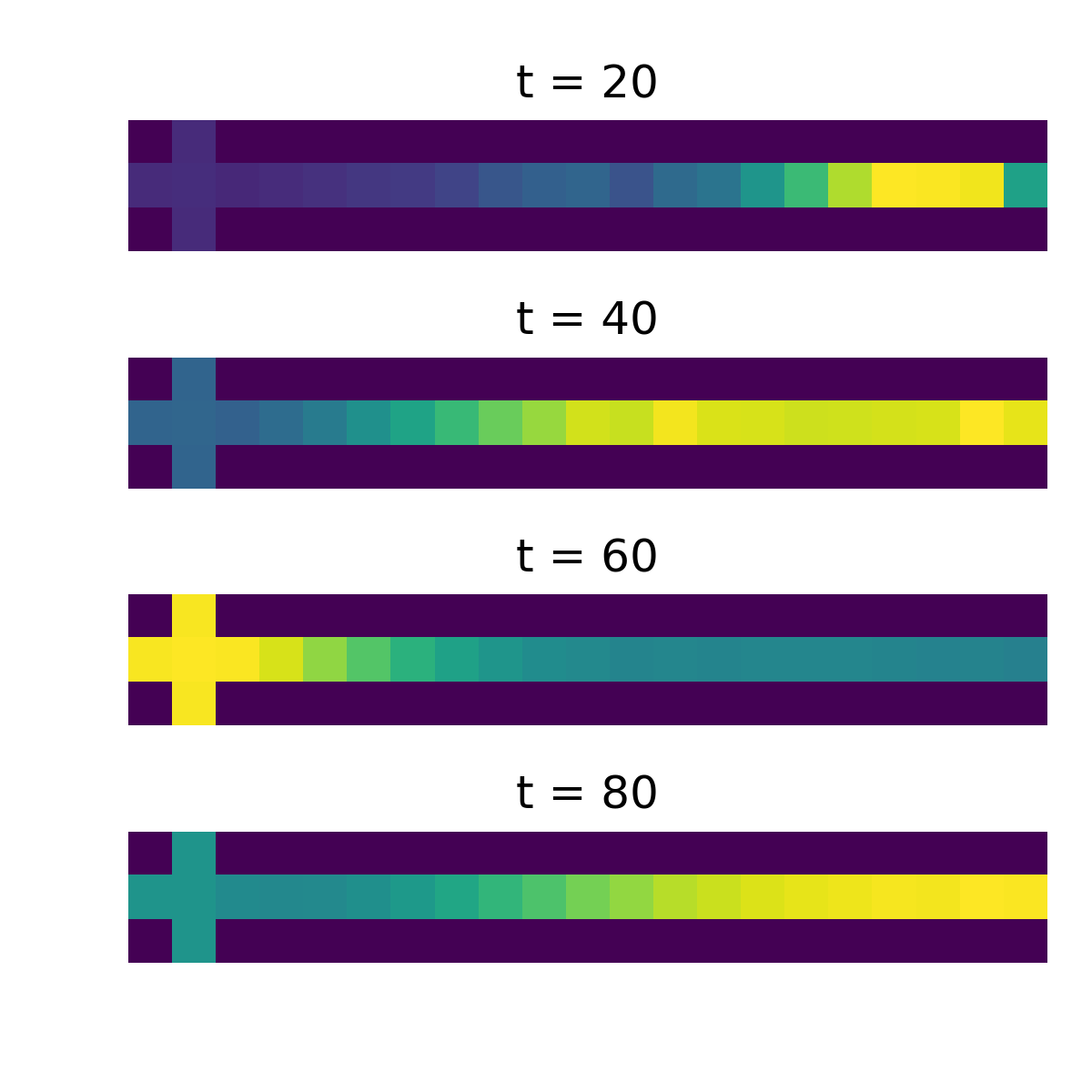}
        \vspace{-2.5em}
        \caption{Oscillatory Learning\label{fig:gridworld-oscillation}}
    \end{subfigure}%
    \vspace{-1em}
    \caption{\figleft \; State Marginals of various exploration methods. \figright \; Without historical averaging, two-player games exhibit oscillatory learning dynamics.}
    \vspace{-0.5em}
\end{figure}

\begin{figure}[t]
    \centering
    \includegraphics[width=\linewidth,trim=20pt 14pt 14pt 14pt,clip]{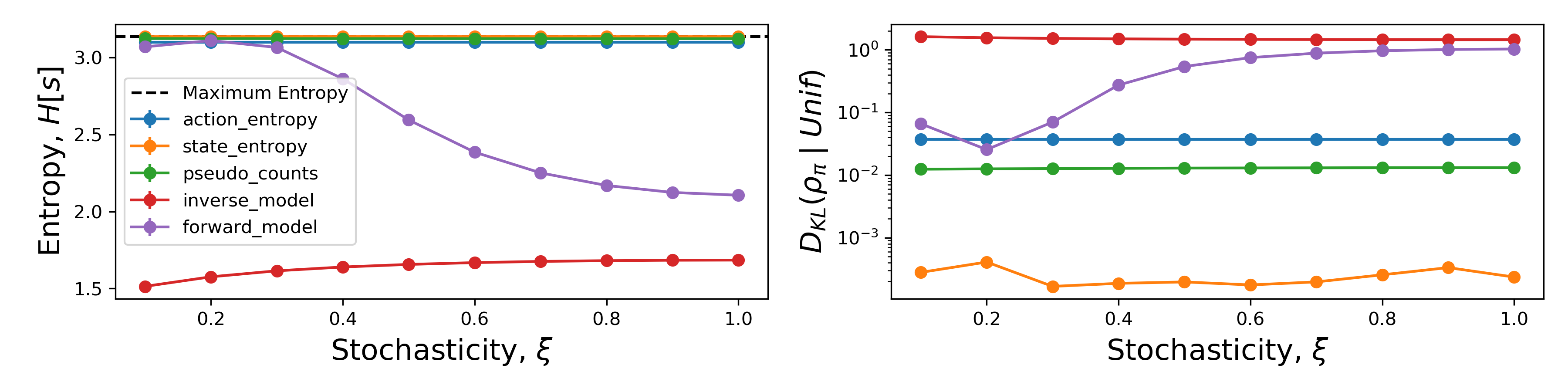}
    \vspace{-2em}
    \caption{\textbf{Effect of Environment Stochasticity on Exploration}: We record the amount of exploration in the didactic gridworld environment as we increase the stochasticity of the dynamics. Both subplots were obtained from the same trajectory data.}
    \label{fig:stochasticity}
    \vspace{-1.5em}
\end{figure}

In this section, we build intuition for why SMM is an important improvement on top of existing exploration methods, and why historical averaging is an important ingredient in maximizing the SMM objective. We will consider the gridworld shown in Fig~\ref{fig:marginals}. In each state, the agent can move up/down/left/right. In most states the commanded action is taken with probability 0.1; otherwise a random action is taken. The exception is a ``noisy TV'' state at the intersection of the two hallways, where the stochasticity is governed by a hyperparameter $\xi \in [0, 1]$. The motivation for considering this simple environment is that we can perform value iteration and learn forward/inverse/density models exactly, allowing us to observe the behavior of exploration strategies in the absence of function approximation error.

In our first experiment, we examine the asymptotic behavior of four methods: SMM (state entropy), inverse models, forward models, count-based exploration, and MaxEnt RL (action entropy). Fig.~\ref{fig:marginals} shows that while SMM converges to a uniform distribution over states, other exploration methods are biased towards visiting the stochastic state on the left. To further understand this behavior, we vary the stochasticity of this state and plot the marginal state entropy of each method, which we compute exactly via the power method. Fig.~\ref{fig:stochasticity} shows that SMM achieves high state entropy in all environments, whereas the marginal state entropy of the inverse model \emph{decreases} as the environment stochasticity increases. The other methods fail to achieve high state entropy for all environments.

Our second experiment examines the role of historical averaging (HA). Without HA, we would expect that exploration methods involving a two-player game, such as SMM and predictive-error exploration, would exhibit oscillatory learning dynamics. Fig.~\ref{fig:gridworld-oscillation} demonstrates this: without HA, the policy player and density player alternate in taking actions towards and placing probability mass on the left and right halves of the environment. Recalling that Fig.~\ref{fig:marginals} included HA for SMM, we conclude that HA is an important ingredient for preventing oscillatory learning dynamics.

In summary, this didactic experiment illustrates that prior methods fail to perform uniform exploration, and that historical averaging is important for preventing oscillation. Our next experiments will show that SMM also accelerates exploration on complex, high-dimensional tasks.

\section{Experimental Evaluation}

\begin{figure}[t]
\centering
    \begin{subfigure}[b]{0.5\linewidth}
        \centering
        \includegraphics[height=2cm]{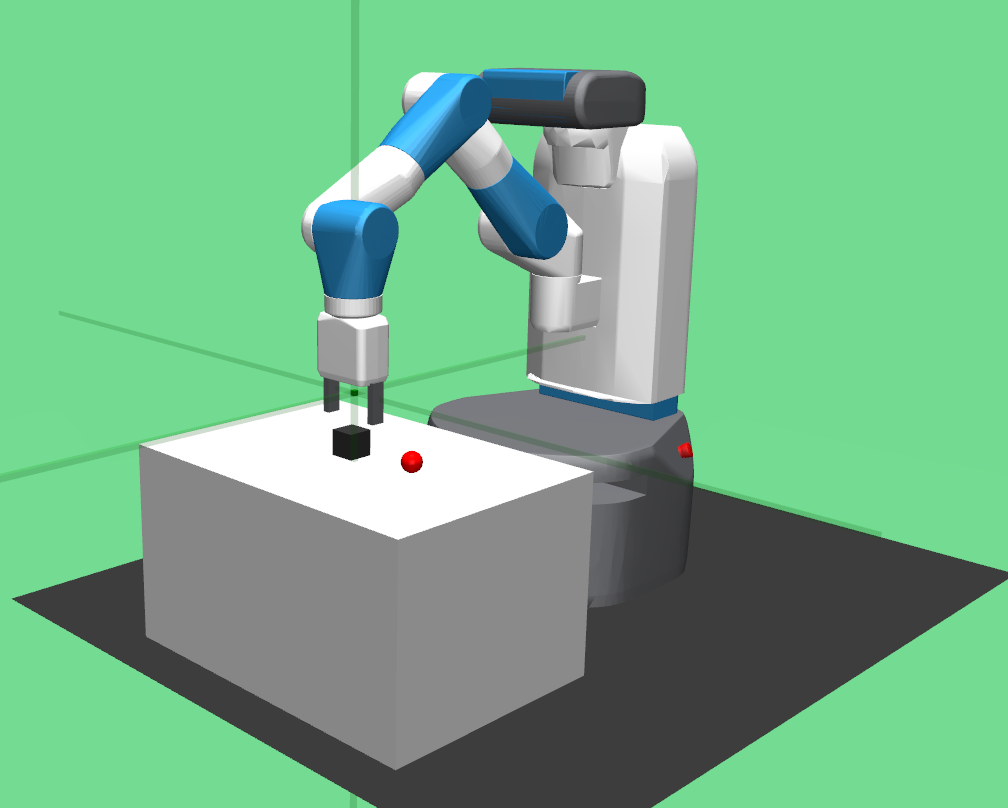}
        \caption{\emph{Fetch} environment\label{fig:fetch}}
    \end{subfigure}%
    \hfill
    \begin{subfigure}[b]{0.5\linewidth}
        \centering
        \includegraphics[height=2cm]{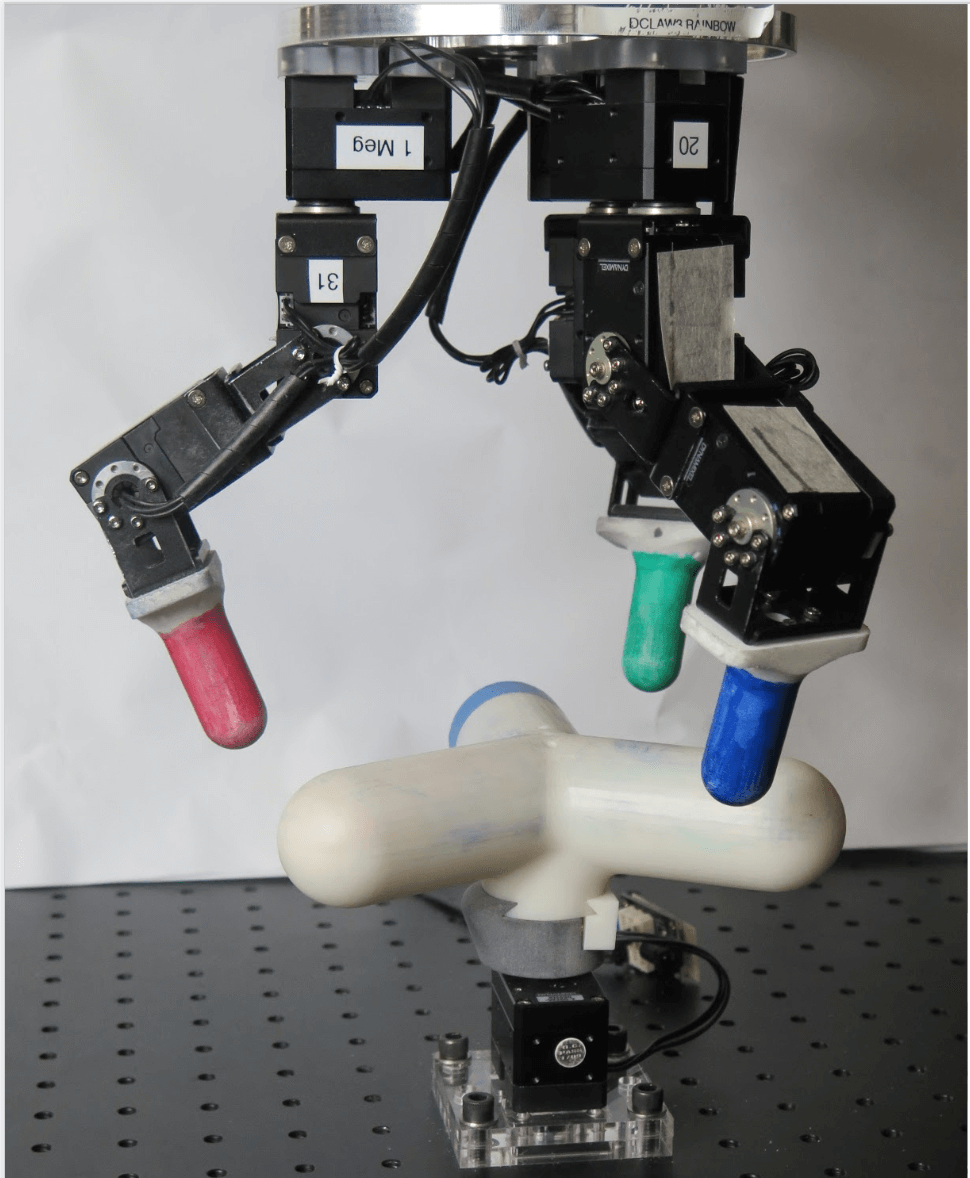}
        \caption{\emph{D'Claw} robot\label{fig:dclaw}}
    \end{subfigure}%
    \vspace{-1em}
    \caption{Simulated and real-world manipulation environments.}
    \vspace{-1.5em}
\end{figure}

In this section, we empirically study whether our method learns to explore effectively when scaled to more complex RL benchmarks, 
and compare against prior exploration methods. Our experiments demonstrate how State Marginal Matching provides good exploration, a key component of which is the historical averaging step.

\textbf{Environments}: We ran manipulation experiments in both a simulated \emph{Fetch} environment~\citep{plappert2018multi} consisting of a single gripper arm and a block object on top of the table (Fig.~\ref{fig:fetch}), as well as on a real-world \emph{D'Claw}~\citep{ahn2019robel} robot, which is a 3-fingered hand positioned vertically above a handle that it can turn (Fig.~\ref{fig:dclaw}).
In the \emph{Fetch} environment~\citep{plappert2018multi}, we defined the target distribution to be uniform over the entire state space (joint + block configuration), with the constraints that we put low probability mass on states where the block has fallen off the table; that actions should be small; and that the arm should be close to the object. For all experiments on the \emph{D'Claw} robot, we used a target distribution that places uniform mass over all object angles~$[-180^\circ, 180^\circ]$.

\begin{figure}[t]
\centering
    \centering
    \includegraphics[width=\columnwidth,trim=0 0 0 40pt,clip]{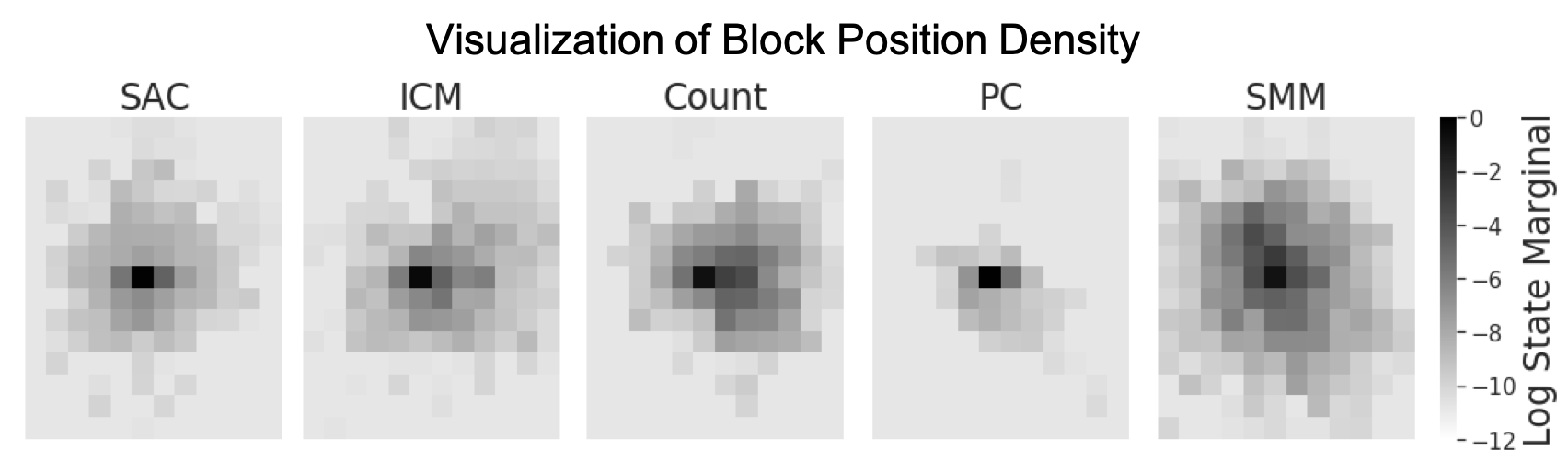}
    \vspace{-2.5em}
\caption{After training, we visualize the policy's log state marginal over the object coordinates in \emph{Fetch}. %
SMM achieves wider state coverage than baselines.\label{fig:visualization-algos}}
\vspace{-1.3em}
\end{figure}

\begin{figure*}
    \centering
    \vspace{-0.3em}
    \begin{subfigure}[b]{0.29\textwidth}
    \includegraphics[align=c,width=\textwidth]{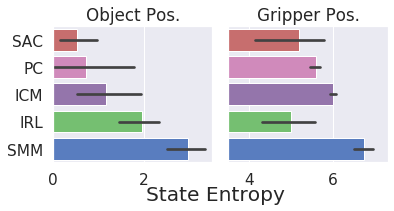}
    \vspace{-0.2em}
    \caption{\emph{Fetch} environment\label{fig:simulated-manipulation-state-entropy}}
    \end{subfigure}
    \begin{subfigure}[b]{0.35\textwidth}
    \centering
    \includegraphics[align=c,width=0.52\textwidth]{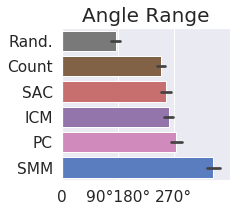}
    \includegraphics[align=c,width=0.46\textwidth]{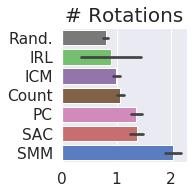}
    \vspace{-0.2em}
    \caption{Sim2Real on \emph{D'Claw}\label{fig:dclaw-min-max-angle}}
    \end{subfigure}%
    \begin{subfigure}[b]{0.35\textwidth}
    \includegraphics[align=c,width=\textwidth]{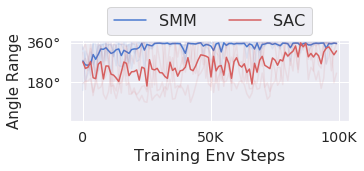}
    \vspace{-0.2em}
    \caption{Training on Hardware (\emph{D'Claw})\label{fig:dclaw-angle}}
    \end{subfigure}
    \vspace{-1em}
\caption{\textbf{The Exploration of SMM}:
\textbf{(a)}\; In the \emph{Fetch} environment, we plot the policy's state entropy over the object and gripper coordinates, averaged over 1,000 epochs. %
SMM explores more than baselines, as indicated by the larger state entropy (larger is better).
\textbf{(b)} \; In the \emph{D'Claw} environment, we trained policies in simulation and then observed how far the trained policy rotated the knob on the hardware robot, measuring both the total number of rotations and the minimum and maximum valve rotations. SMM turns the knob further to the left and right than the baselines, and also completes a larger cumulative number of rotations. \textbf{(c)}\; We trained SAC and SMM on the real robot for 1e5 environment steps (about 9 hours in real time), and measured the angle turned throughout training. We see that SMM moves the knob more and visits a wider range of states than SAC. All results are averaged over 4-5 seeds.}
\end{figure*}

\textbf{Baselines}: We compare to a state-of-the-art off-policy MaxEnt RL algorithm, Soft Actor-Critic (SAC)~\citep{haarnoja2018soft}; an inverse RL algorithm, Generative Adversarial Imitation Learning (GAIL)~\citep{ho2016generative}; and three exploration methods:
\vspace{-1em}
\begin{itemize}[noitemsep, leftmargin=*]
    \item Count-based Exploration (Count), which discretizes states and uses $-\log \hat{\pi}(s)$ as an exploration bonus.
    \item Pseudo-counts (PC)~\citep{bellemare2016unifying}, which uses the recoding probability as a bonus.
    \item Intrinsic Curiosity Module (ICM)~\citep{pathak2017curiosity}, which uses prediction error as a bonus.
\end{itemize}
\vspace{-0.5em}
All exploration methods have access to exactly the same information and the same extrinsic reward function. SMM interprets this extrinsic reward as the log probability of a target distribution: $p^*(s) \propto \exp(r_\text{env}(s))$.  We used SAC as the base RL algorithm for all exploration methods (SMM, Count, PC, ICM). We use a variational autoencoder (VAE) to model the density $q(s)$ for both SMM and Pseudocounts. For SMM, we approximate the historical average of density models (Eq.~\ref{eq:density-2}) with the most recent iterate, and use a uniform categorical distribution for the prior $p(z)$. %
To train GAIL, we generated synthetic expert data by sampling expert states from the target distribution $p^*(s)$ (see Appendix~\ref{section:gail} for details). 
Results for all experiments are averaged over 4-5 random seeds. Additional details about the experimental setup can be found in Appendix~\ref{appendix:implementation-details}.

\vspace{-0.5em}
\subsection{State Coverage at Convergence}

In the \emph{Fetch} environment, we trained each method for 1e6 environment steps and then measured how well they explore by computing the marginal state entropy, which we compute by discretizing the state space.\footnote{Discretization is used only for evaluation, no policy has access to it (except for Count).} In Fig.~\ref{fig:simulated-manipulation-state-entropy}, we see that SMM maximizes state entropy at least as effectively as prior methods, if not better. While this comparison is somewhat unfair, as we measure exploration using the objective that SMM maximizes, none of the methods we compare against propose an alternative metric for exploration.%

On the \emph{D'Claw} robot, we trained SMM and other baselines in simulation, and then evaluated the acquired exploration policy on the real robot using two metrics: the total number of rotations (in either direction), and the maximum radians turned (in both directions). For each method, we computed the average metric across 100 evaluation episodes. We repeated this process for 5 independent training runs. Compared to the baselines, SMM turns the knob more to a wider range of angles (Fig.~\ref{fig:dclaw-min-max-angle}). To test for statistical significance, we used a 1-sided Student's t-test to test the hypothesis that SMM turned the knob more to a wider range of angles than SAC. The p-values were all less than 0.05: $p = 0.046$ for number of rotations, $p = 0.019$ for maximum clockwise angle, and $p = 0.001$  for maximum counter-clockwise angle. The results on the \emph{D'Claw} hardware robot suggests that exploration techniques may actually be useful in the real world, which may encourage future work to study exploration methods on real-world tasks.

We also investigated whether it was possible to learn an exploration policy directly in the real world, without the need for a simulator, an important setting in scenarios where faithful simulators are hard to build.
In Fig.~\ref{fig:dclaw-angle}, we plot the range of angles that the policy explores \emph{throughout} training. Not only does SMM explore a wider range of angles than SAC, but its ability to explore increases throughout training, suggesting that the SMM objective is correlated with real-world metrics of exploration.

\vspace{-0.5em}
\subsection{Test-time Exploration}
\vspace{-0.5em}

\begin{figure}
    \centering
    \begin{subfigure}[b]{0.48\columnwidth}
        \includegraphics[width=\textwidth]{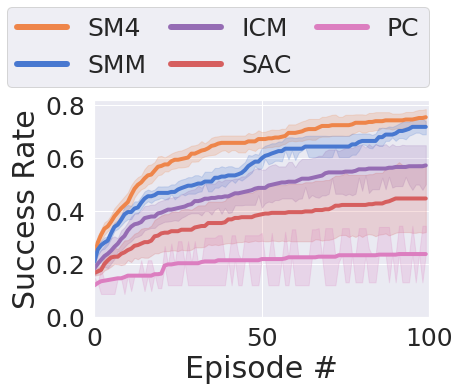}
    \vspace{-1.5em}
        \caption{Test-time adaptation\label{fig:fetch-test-algos}}
    \end{subfigure}
    \begin{subfigure}[b]{0.51\columnwidth}
    \includegraphics[width=\textwidth,trim=150pt 0 0 0,clip]{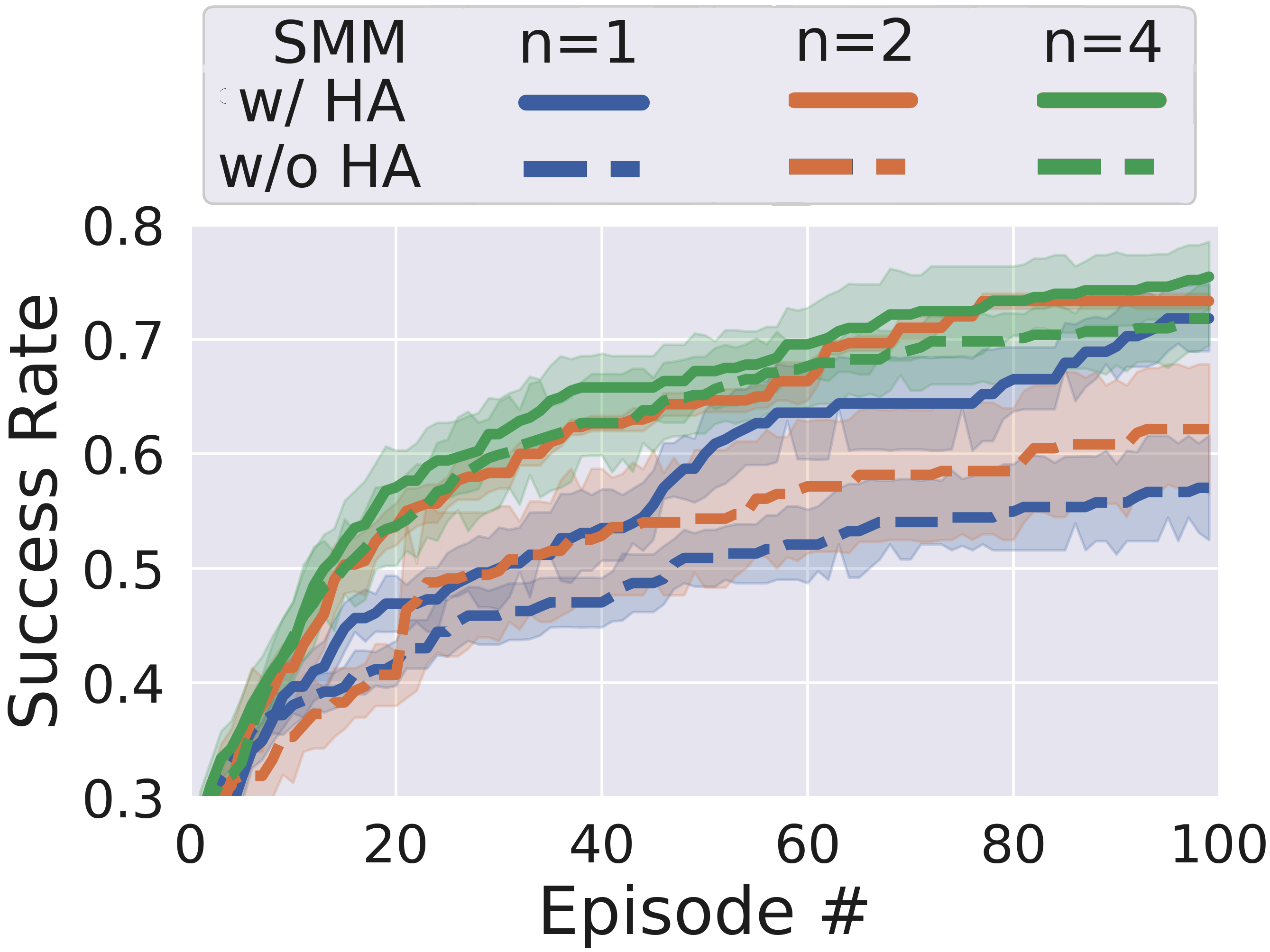}
    \vspace{-1em}
    \caption{SMM ablation\label{fig:manipulation-mixture}}
    \end{subfigure}%
    \vspace{-1.1em}
        \caption{\textbf{Fast Adaptation}: \textbf{(a)}\; We plot the percentage of test-time goals found within $N$ episodes. SMM and its mixture-model variant SM4 both explore faster than the baselines, allowing it to successfully find the goal in fewer episodes. \textbf{(b)}\; We compare SMM/SM4 with different numbers of mixtures, and with vs. without historical averaging. Increasing the number of latent mixture components $n \in \{1, 2, 4\}$ accelerates exploration, as does historical averaging. Error bars show std.\ dev.\ across 4 random seeds.\label{fig:fetch-test}}
    \vspace{-1.7em}
\end{figure}

We also evaluated whether the exploration policy acquired by SMM allows us to solve downstream tasks more quickly. 
As shown in Fig.~\ref{fig:fetch-test-algos}, SMM and its mixture variant, SM4, both adapt substantially more quickly than other exploration methods, achieving a success rate 20\% higher than the next best method, and reaching the same level of performance of the next baseline (ICM) in 4x fewer episodes.

\textbf{Ablation Study}.\; In Fig.~\ref{fig:manipulation-mixture}, we study the effect of mixture modelling on test-time exploration. %
After running SMM/SM4 with a uniform distribution, we count the number of episodes required to find an (unknown) goal state. We run each method for the same number of environment transitions; a mixture of three policies \emph{does not} get to take three times more transitions. We find that increasing the number of mixture components increases the agents success. However, the effect was smaller when using historical averaging. Taken together, this result suggests that efficient exploration requires \emph{either} historical averaging \emph{or} mixture modelling, but might not need both. In particular, SMM without historical averaging attains similar performance as the next best baseline (ICM), suggesting that historical averaging is the key ingredient, while the particular choice of prediction error or VAE is less important.

\vspace{-0.5em}
\section{Related Work}
\vspace{-0.5em}

Many exploration algorithms can be classified by whether they explore in the space of actions, policy parameters, goals, or states.
Common exploration strategies including $\epsilon$-greedy and Ornstein–Uhlenbeck noise~\citep{lillicrap2015continuous}, and MaxEnt RL algorithms~\citep{ziebart2010modeling,haarnoja2018soft} explore in the action space. \citet{fortunato2017noisy, plappert2017parameter} show that adding parameter noise to the policy can result in good exploration.

Most closely related to our work are methods that perform exploration in the space of states or goals~\citep{colas2018curious,held2017automatic,nair2018visual,pong2019skew, hazan2018provably}. While the state marginal matching objective is also considered in~\citet{hazan2018provably}, our work builds upon this prior work in a number of dimensions. First, we explain how to do distribution matching properly by analyzing the SMM objective as a two-player game and applying historical averaging from fictitious play. Our analysis also leads to a unified view of a large class of existing intrinsic motivation techniques that previously were proposed as exploration heuristics, showing that in fact these methods \emph{almost} perform state marginal matching. Furthermore, we introduce the notion that this objective yields a task-agnostic ``policy prior'' that can quickly solve new tasks, and demonstrate this empirically on complex RL benchmarks. We also prove that the SMM objective induces the optimal exploration for a certain class of goal-reaching tasks (Appendix~\ref{appendix:hitting-time}).

One class of exploration methods uses prediction error of some auxiliary task as an exploration bonus, which provides high (intrinsic) reward in states where the predictive model performs poorly~\citep{pathak2017curiosity,oudeyer2007intrinsic,schmidhuber1991possibility,houthooft2016vime,burda2018exploration}. Another set of approaches~\citep{tang2017exploration,bellemare2016unifying,schmidhuber2010formal} directly encourage the agent to visit novel states. While all methods effectively explore during the course of solving a single task~\citep{taiga2019benchmarking}, we showed in Sec.~\ref{sec:prediction-error} that the policy obtained at convergence is often not a good exploration policy by itself. In contrast, our method converges to a highly-exploratory policy by maximizing state entropy. %

The problems of exploration and meta-RL are tightly coupled. 
Meta-RL algorithms~\citep{duan2016rl, finn2017model, rakelly2019efficient, mishra2017simple} must perform effective exploration if they hope to solve a downstream task. Some prior work has explicitly looked at the problem of learning to explore~\citep{gupta2018meta, xu2018learning}. 
Our problem statement is similar to meta-learning, in that we also aim to learn a policy as a prior for solving downstream tasks. However, whereas meta-RL requires a distribution of task reward functions, our method requires only a single target state marginal distribution. Due to the simpler problem assumptions and training procedure, our method may be easier to apply in real-world domains.

Related to our approach are maximum \emph{action} entropy algorithms~\citep{haarnoja2018soft, kappen2012optimal, rawlik2013stochastic, ziebart2008maximum, theodorou2012relative}.
While these algorithms are referred to as \emph{MaxEnt RL}, they are maximizing entropy over actions, not states. These algorithms can be viewed as performing inference on a graphical model where the likelihood of a trajectory is given by its exponentiated reward~\citep{toussaint2006probabilistic,levine2018reinforcement,abdolmaleki2018maximum}.
While distributions over trajectories induce distributions over states, computing the exact relationship requires integrating over all possible trajectories, an intractable problem for most MDPs.
A related but distinct class of \emph{relative} entropy methods use a similar entropy-based objective to limit the size of policy updates~\citep{peters2010relative, schulman2015trust}.

Many of the underlying ingredients of our method, such as adversarial games and density estimation, have seen recent progress in imitation learning~\citep{ziebart2008maximum, ho2016generative, finn2016guided, fu2017learning}. Similar to some inverse RL algorithms~\citep{ho2016generative,fu2018learning}, our method iterates between learning a policy and learning a reward function, though our reward function is obtained via a density model instead of a discriminator. While inverse RL algorithms assume access to expert trajectories, we instead assume access to the density of the target state marginal distribution. In many realistic settings, such as robotic control with many degrees of freedom, providing fully-specified trajectories may be much more challenging than defining a target state marginal distribution. The latter only requires some aggregate statistics about expert behavior, and does not even need to be realizable by any policy.

\vspace{-0.5em}
\section{Conclusion}

This paper studied state marginal matching as a formal objective for exploration. %
While it is often unclear what existing exploration methods will converge to,
the SMM objective has a clear solution: at convergence, the policy should visit states in proportion to their density under a target distribution.
The resulting policy can be used as a prior in a multi-task setting to amortize exploration and adapt more quickly to new, potentially sparse, reward functions. 

We explain how to perform distribution matching properly via historical averaging. We further demonstrate that prior work approximately maximizes the SMM objective, offering an explanation for the success of these methods. 
Augmenting these prior methods with an important historical averaging step not only guarantees that they converge, but also empirically improves their exploration. 
Experiments on both simulated and real-world tasks demonstrated how SMM learns to explore, enabling an agent to efficiently explore in new tasks provided at test time. 

In summary, our work unifies prior exploration methods as performing approximate distribution matching, and explains how state distribution matching can be performed properly. This perspective provides a clearer picture of exploration, and is useful particularly because many of the underlying ingredients, such as adversarial games and density estimation, have seen recent progress and therefore might be adopted to improve exploration methods.

\paragraph{Acknowledgements}
{\footnotesize
We thank Michael Ahn for help with running experiments on the \emph{D'Claw} robots. We thank Maruan Al-Shedivat, Danijar Hafner, and Ojash Neopane for helpful discussions and comments. LL is supported by NSF grant DGE-1745016 and AFRL contract FA8702-15-D-0002. BE is supported by Google. EP is supported by ONR grant N000141812861 and Apple. RS is supported by NSF grant IIS1763562, ONR grant N000141812861, AFRL CogDeCON, and Apple. Any opinions, findings and conclusions or recommendations expressed in this material are those of the author(s) and do not necessarily reflect the views of NSF, AFRL, ONR, Google or Apple. We also thank Nvidia for their GPU support. 
}

\clearpage
\appendix

\begin{figure}[t]
\centering
 \begin{algorithm}[H]
    \caption{State Marginal Matching with Mixtures of Mixtures (SM4) \label{alg:smm-mop}}
  \begin{algorithmic}
        \STATE \textbf{Input:} Target distribution $p^*(s)$
        \STATE Initialize policy $\pi_z(a \mid s)$, density model $q_z(s)$, discriminator $d(z \mid s)$, and replay buffer $\mathcal{B}$.
\WHILE{not converged}
\FOR{$z = 1, \cdots, n$} %
    \STATE $q_z^{(m)} \gets \argmax_q \mathbb{E}_{\{s \mid (z', s) \sim \mathcal{B}^{(m-1)},z' = z\}} \left[ \log q(s) \right]$
\ENDFOR
\STATE $d^{(m)} \gets \argmax_d \mathbb{E}_{(z, s) \sim \mathcal{B}^{(m-1)}} \left[ \log d(z \mid s) \right]$ \COMMENT{(2) Update discriminator.}
\FOR{$z = 1, \cdots, n$} 
    \STATE $r_z^{(m)}(s) \triangleq \log p^*(s) -\log q^{(m)}_z(s) + \log d^{(m)}(z \mid s) - \log p(z)$
    \STATE $\pi_z^{(m)} \gets \argmax_\pi \E_{\rho_{\pi}(s)} \left[ r_z^{(m)}(s) \right]$ %
\ENDFOR
\STATE Sample latent skill $z^{(m)} \sim p(z)$
\STATE Sample transitions $\{(s_t, a_t, s_{t+1})\}_{t=1}^T$ with $\pi_z^{(m)}(a \mid s)$ %
\STATE $\mathcal{B}^{(m)} \gets \mathcal{B}^{(m-1)} \cup \{ (z^{(m)}, s_t,a_t,s_{t+1}) \}_{t=1}^T$
\ENDWHILE
\STATE \textbf{return} $\{\{\pi_1^{(1)}, \cdots, \pi_n^{(1)}\}, \cdots, \{\pi_1^{(m)}, \cdots, \pi_n^{(m)}\}\}$
\end{algorithmic} 
\end{algorithm}
\vspace{-1em}
\caption*{Alg.~\ref{alg:smm-mop}.
{\small An algorithm for learning a \emph{mixture} of policies $\pi_1, \pi_2, \cdots, \pi_n$ that do state marginal matching \emph{in aggregate}. The algorithm (1) fits a density model $q^{(m)}_z(s)$ to approximate the state marginal distribution for each policy $\pi_z$; (2) learns a discriminator $d^{(m)}(z \mid s)$ to predict which policy $\pi_z$ will visit state $s$; and (3) uses RL to update each policy $\pi_z$ to maximize the expected return of its corresponding reward function derived in Eq.~\ref{eq:smm-mop-objective}. %
In our implementation, the density model $q_z(s)$ is a VAE that inputs the concatenated vector $\{s, z\}$ of the state $s$ and the latent skill $z$ used to obtain this sample $s$; and the discriminator is a feedforward MLP.
The algorithm returns the historical average of mixtures of policies (a total of $n\cdot m$ policies).
}
}
\vspace{-1em}
\end{figure}

\section{Proofs}\label{sec:proofs}

\begin{proof}[Proof of Proposition~\ref{lemma:max-min-equivalence}]
Note that the objective in Eq.~\ref{eq:min-max-obj} can be written as
\begin{equation*}
    \E_{\rho_{\pi}(s)}[\log p^*(s) - \log \rho_{\pi}(s)] + \kl{\rho_{\pi}(s)}{q(s)}.
\end{equation*}
By Assumption~\ref{assumption-existence}, $\kl{\rho_{\pi}(s)}{q(s)} = 0$ for some $q \in Q$, so we obtain the desired result:
\begin{align*}
    &\max_\pi \left( \min_q \E_{\rho_{\pi}(s)}[\log p^*(s) -\log q(s)] \right)\\
    = &\max_\pi \big( \E_{\rho_{\pi}(s)}[\log p^*(s) - \log \rho_{\pi}(s)] + \\
    & \qquad\quad\min_q \kl{\rho_{\pi}(s)}{q(s)} \big) \\
    = &\max_\pi \E_{\rho_{\pi}(s)}[\log p^*(s) - \log \rho_{\pi}(s)]. \qedhere
\end{align*}
\end{proof}

\section{Choosing $p^*(s)$ for Goal-Reaching Tasks}
\label{appendix:hitting-time}

In general, the choice of the target distribution $p^*(s)$ will depend on the distribution of test-time tasks. In this section, we consider the special case where the test-time tasks correspond to goal-reaching derive the optimal target distribution $p^*(s)$.
We consider the setting where goals $g \sim p_g(g)$ are sampled from some known distribution. Our goal is to minimize the number of episodes required to reach that goal state. We define reaching the goal state as visiting a state that lies within an $\epsilon$ ball of the goal, where both $\epsilon > 0$ and the distance metric are known.

We start with a simple lemma that shows that the probability that we reach the goal at any state in a trajectory is at least the probability that we reach the goal at a randomly chosen state in that same trajectory. Defining the binary random variable $z_t \triangleq \mathbbm{1}(\|s_t - g\| \le \epsilon)$ as the event that the state at time $t$ reaches the goal state, we can formally state the claim as follows:
\begin{lemma} 
\begin{equation*}
    p \left(\sum_{t=1}^T z_t > 0 \right) \ge p(z_\mathbf{t}) \qquad \text{where} \quad \mathbf{t} \sim Unif[1, \cdots, H]
\end{equation*}
\label{lemma:1}
\end{lemma}
\begin{proof}
We start by noting the following implication:
\begin{equation*}
z_{\mathbf{t}} = 1 \implies \sum_{t=1}^T z_t > 0
\end{equation*}
Thus, the probability of the event on the RHS must be at least as large as the probability of the event on the LHS:
\begin{equation*}
   p(z_\mathbf{t}) \le p \left(\sum_{t=1}^T z_t > 0 \right)
\end{equation*}
\end{proof}
Next, we look at the expected number of \emph{episodes} to reach the goal state. Since each episode is independent, the expected hitting time is simply
\begin{align*}
   \textsc{HittingTime}(s)
   &= \frac{1}{p(\text{some state reaches $s$})} \\
   &= \frac{1}{p \left(\sum_{t=1}^T z_t > 0 \right)} \le \frac{1}{p(z_\mathbf{t})}
\end{align*}
Note that we have upper-bounded the hitting time using Lemma~\ref{lemma:1}.
Since the goal $g$ is a random variable, we take an expectation over $g$:
\begin{align*}
    &\E_{s \sim p_g(s)}\left[\textsc{HittingTime}(s)\right]
    \le \E_{s \sim p_g(s)}\left[ \frac{1}{p(z_\mathbf{t})}\right] \\
    &\quad\le \E_{s \sim p_g(s)}\left[ \frac{1}{\int p^*(\tilde{s}) \mathbbm{1}(\|s - \tilde{s}\| \le \epsilon)d \tilde{s}}\right] \triangleq \mathcal{F}(p^*)
\end{align*}
where $p^*(s)$ denotes the target state marginal distribution. 
We will minimize $\mathcal{F}$, an upper bound on the expected hitting time.
\begin{lemma}
The state marginal distribution $p^*(s) \propto \sqrt{\tilde{p}(s)}$ minimizes $\mathcal{F}(p^*)$, where \mbox{$\tilde{p}(s) \triangleq \int p_g(\tilde{s}) \mathbbm{1}(\|s - \tilde{s}\| \le \epsilon) d\tilde{s}$} is a smoothed version of the target density.
\label{lemma-2}
\end{lemma}
Before presenting the proof, we provide a bit of intuition. In the case where $\epsilon \rightarrow 0$, the optimal target distribution is $p^*(s) \propto \sqrt{p_g(s)}$. For non-zero $\epsilon$, the policy in Lemma~\ref{lemma-2} is equivalent to convolving $p_g(s)$ with a box filter before taking the square root. In both cases, we see that the optimal policy does distribution matching to some function of the goal distribution. Note that $\tilde{p}(\cdot)$ may not sum to one and therefore is not a proper probability distribution.
\begin{proof}
We start by forming the Lagrangian:
\begin{align*}
    \mathcal{L}(p^*) &\triangleq \int \frac{p_g(s)}{\int p^*(\tilde{s}) \mathbbm{1}(\|s - \tilde{s}\| \le \epsilon) \; d \tilde{s}}\; ds \\
    &\qquad + \lambda \left(\int p^*(\tilde{s}) \; d \tilde{s} - 1 \right)
\end{align*}
The first derivative is
\begin{equation*}
    \frac{d \mathcal{L}}{d p^*(\tilde{s})} = \int \frac{-p_g(s)\mathbbm{1}(\|s - \tilde{s}\| \le \epsilon)}{{p^*}^2(\tilde{s})} ds + \lambda = 0
\end{equation*}
Note that the second derivative is positive, indicating that this Lagrangian is convex, so all stationary points must be global minima:
\begin{equation*}
    \frac{d^2 \mathcal{L}}{d p^*(\tilde{s})^2} = \int \frac{2p_g(s)\mathbbm{1}(\|s - \tilde{s}\| \le \epsilon)}{{p^*}^3(\tilde{s})} ds > 0
\end{equation*}
Setting the first derivative equal to zero and rearranging terms, we obtain
\begin{equation*}
    \pi(\tilde{s}) \propto \sqrt{\int p_g(s) \mathbbm{1}(\|s - \tilde{s}\| \le \epsilon) ds}
\end{equation*}
Renaming $\tilde{s} \leftrightarrow s$, we obtain the desired result.
\end{proof}

\subsection{Connections to Goal-Conditioned RL}
\label{appendix:rl-goals}
Goal-Conditioned RL~\citep{kaelbling1993learning,nair2018visual,held2017automatic} can be viewed as a special case of State Marginal Matching when the goal-sampling distribution is learned jointly with the policy.
In particular, consider the State Marginal Matching with a mixture policy (Alg.~\ref{alg:smm-mop}), where the mixture component $z$ maps bijectively to goal states. In this case, we learn goal-conditioned policies of the form $\pi(a \mid s, z)$.
Consider the SMM objective with Mixtures of Policies in Eq.~\ref{eq:smm-mop-objective}.
The second term $p(z \mid s)$ is an estimate of which goal the agent is trying to reach, similar to objectives in intent inference~\citep{ziebart2009planning,Xie_2013_ICCV}. The third term $\pi(s \mid z)$ is the distribution over states visited by the policy when attempting to reach goal $z$. For an optimal goal-conditioned policy in an infinite-horizon MDP, both of these terms are Dirac functions:
\begin{equation*}
    \pi(z \mid s) = \rho_\pi(s \mid z) = \mathbbm{1}(s = z)
\end{equation*}
In this setting, the State Marginal Matching objective simply says to sample goals $g \sim \pi(g)$ with probability equal to the density of that goal under the target distribution.
\begin{align*}
\kl{\rho_{\pi}(s)}{p^*(s)} = \E_{\substack{z \sim \pi(z)\\s \sim \pi(s \mid z)}} \left[\log p^*(s) - \log \pi(z) \right]
\end{align*}

Whether goal-conditioned RL is the preferable way to do distribution matching depends on (1) the difficulty of sampling goals and (2) the supervision that will be provided at test time. It is natural to use goal-conditioned RL in settings where it is easy to sample goals, such as when the space of goals is small and finite or otherwise low-dimensional. If a large collection of goals is available apriori, we could use importance sampling to generate goals to train the goal-conditioned policy~\citep{pong2019skew}.
However, many real-world settings have high-dimensional goals, which can be challenging to sample.
While goal-conditioned RL is likely the right approach when we will be given a test-time task, a latent-conditioned policy may explore better in settings where the goal-state is not provided at test-time.

\section{Additional Experiments}

\subsection{Navigation experiments}

\begin{figure*}
    \vspace{-0.5em}
    \centering
     \begin{subfigure}[b]{0.18\textwidth}
        \centering
        \includegraphics[align=c,width=\textwidth]{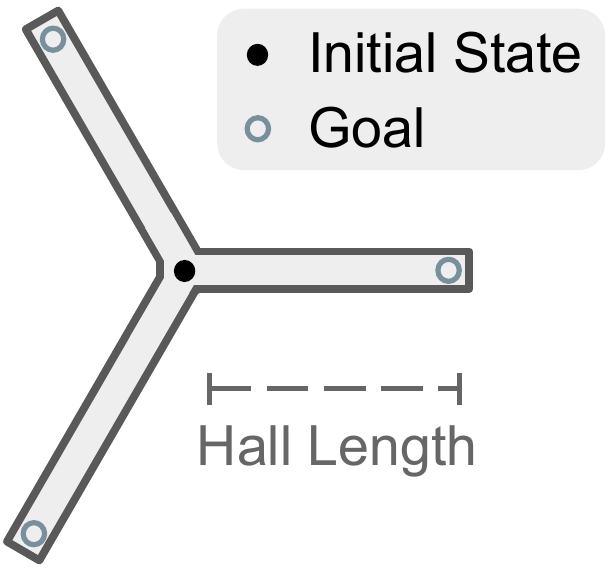}
        \caption{\label{fig:point-2d-env}\emph{Navigation} env.}
    \end{subfigure}%
    ~
    \begin{subfigure}[b]{0.09\textwidth}
        \centering
        \includegraphics[align=c,width=\textwidth,trim=0 0 0 0,clip]{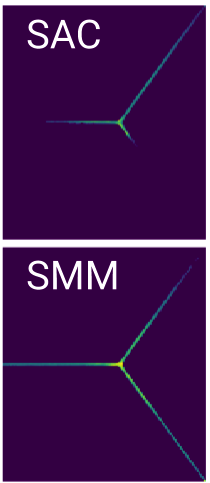}

        \tiny{3 Halls of length 10}
        \vspace{0.1pt}
        \caption{\label{fig:point-2d-density}$\rho_\pi(s)$}
    \end{subfigure}%
    ~
    \begin{subfigure}[b]{0.31\textwidth}
        \centering
        \includegraphics[align=c,width=\textwidth,trim=0 0 0 2pt,clip]{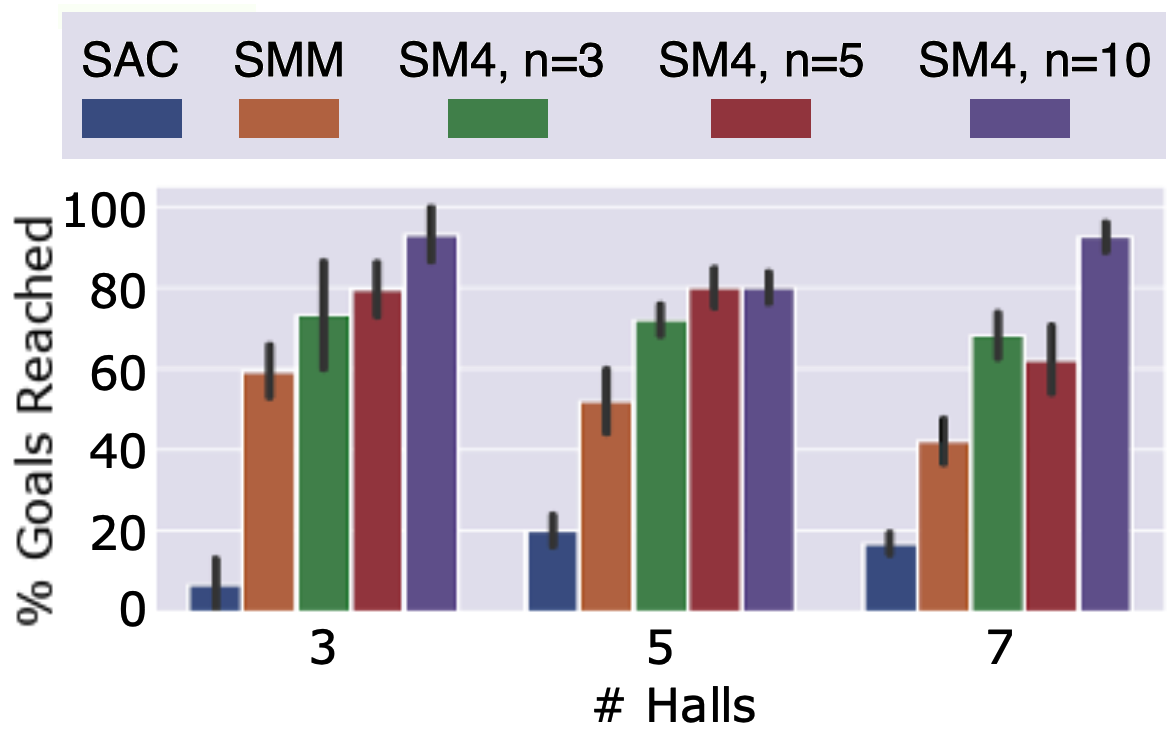}
        \vspace{-5pt}
        \caption{\label{fig:point-2d-results}\% Goals reached during training}
    \end{subfigure}%
    ~
    \begin{subfigure}[b]{0.38\textwidth}
        \centering
    \includegraphics[align=c,width=\textwidth]{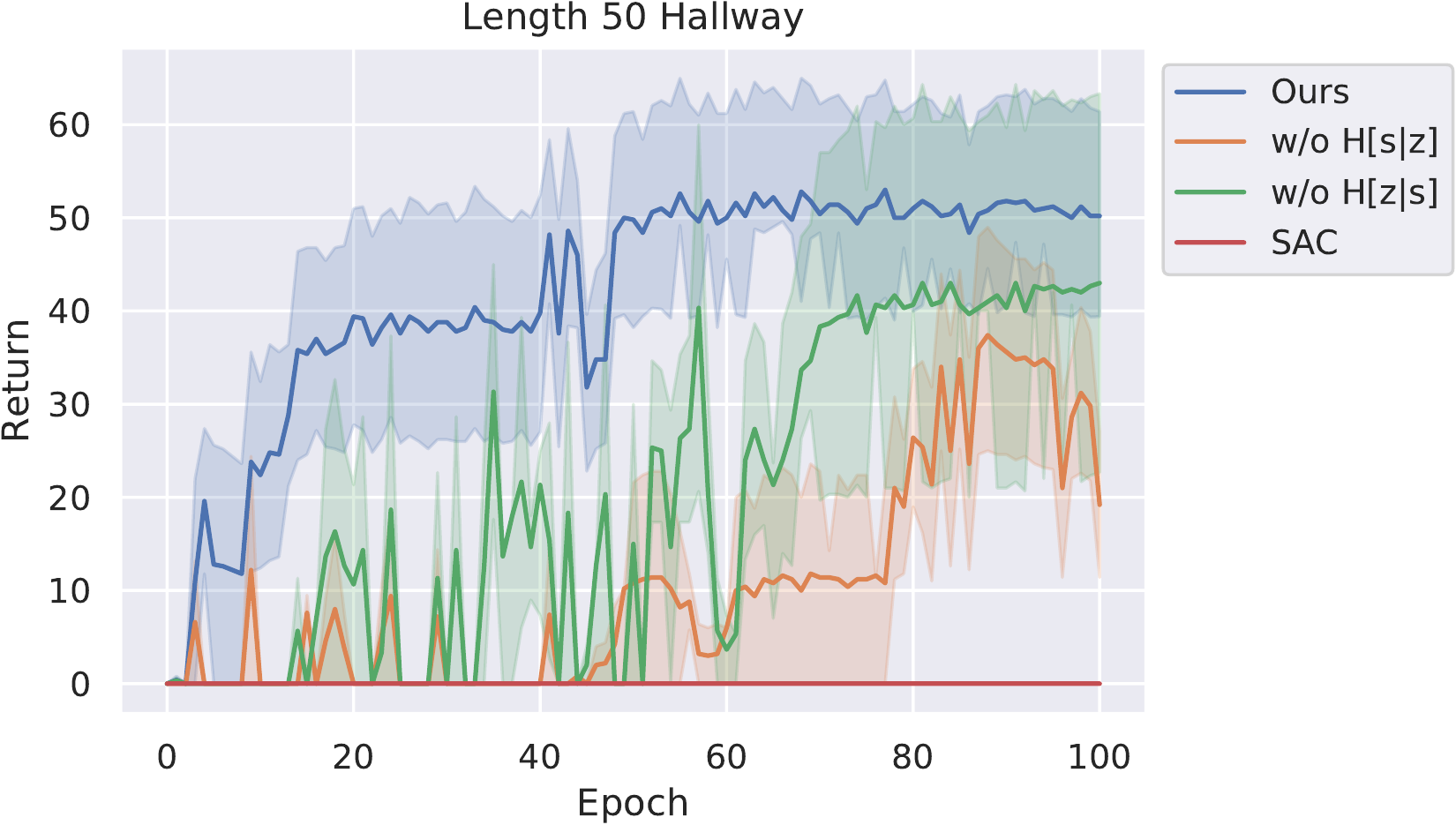}
    \caption{Train-time Performance\label{fig:point-3}}
    \end{subfigure}%
    \caption{\footnotesize\textbf{Exploration in State Space (SMM) vs. Action Space (SAC) for \emph{Navigation}}: \textbf{(a)}:\;~A point-mass agent is spawned at the center of $m$ long hallways that extend radially outward, and {\color{black}
    the target state distribution places uniform probability mass $\frac{1}{m}$ at the end of each hallway.%
    } We can vary the length of the hallway and the number of hallways to control the task difficulty. \textbf{(b)}\; A heatmap showing states visited by SAC and SMM during training illustrates that SMM explores a wider range of states. \textbf{(c)}\; SMM reaches more goals than the MaxEnt baseline. SM4 is an extension of SMM that incorporates mixture modelling with $n>1$ skills (see Appendix~\ref{sec:mixture-policies}), and further improves exploration of SMM.
    \textbf{(d)}\; \textbf{Ablation Analysis of SM4}.  On the \emph{Navigation} task, we compare SM4 (with three mixture components) against ablation baselines that lack conditional state entropy, latent conditional action entropy, or both (i.e., SAC) in the SM4 objective (Eq.~\ref{eq:smm-mop-objective}). We see that both terms contribute heavily to the exploration ability of SM4, but the state entropy term is especially critical.
    \label{fig:point-2d}}
    \vspace{-1em}
\end{figure*}

Is exploration in state space (as done by SMM) better than exploration in action space (as done by MaxEnt RL, e.g., SAC)? To study this question, we implemented a \emph{Navigation} environment, shown in Fig.~\ref{fig:point-2d-env}. 
To evaluate each method, we counted the number of hallways that the agent fully explored (i.e., reached the end) during training. Fig.~\ref{fig:point-2d-density} shows the state visitations for the three hallway environment, illustrating that SAC only explores one hallway, whereas SMM explores all three. Fig.~\ref{fig:point-2d-results} also shows that SMM consistently explores 60\% of hallways, whereas SAC rarely visits more than 20\% of hallways.

\subsection{Does Historical Averaging help other baselines?}

In Fig.~\ref{fig:hist-avg}, we see that historical averaging is not only beneficial to SMM, but also improves the exploration of prior methods. The result further supports our hypothesis that prior exploration methods are approximately optimizing the same SMM objective.

\begin{figure}[t]
    \centering
    \includegraphics[width=\columnwidth,trim=0 0 0 5pt,clip]{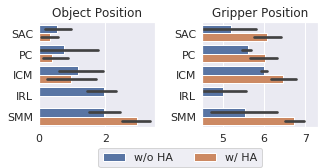}
	\caption{Analysis of historical averaging: After training, we rollout the policy for 1e3 epochs, and record the entropy of the object and gripper positions in \emph{Fetch}. SMM achieves higher state entropy than the other methods. Historical averaging also helps previous exploration methods achieve greater state coverage.\label{fig:hist-avg}}
\end{figure}

\subsection{Non-Uniform Exploration}

We check whether prior knowledge injected via the target distribution is reflected in the policy obtained from State Marginal Matching. Using the same \emph{Fetch} environment as above, we modified the target distribution to assign larger probability to states where the block was on the left half of the table than on the right half. In Fig.~\ref{fig:prior-experiment}, we measure whether SMM is able to achieve the target distribution by measuring the discrepancy between the block's horizontal coordinate and the target distribution. Compared to the SAC baseline, SMM and the Count baseline are half the distance to the target distribution. No method achieves zero discrepancy, suggesting that future methods could be better at matching state marginals.

\begin{figure}[t]
    \centering
    \includegraphics[width=0.49\columnwidth]{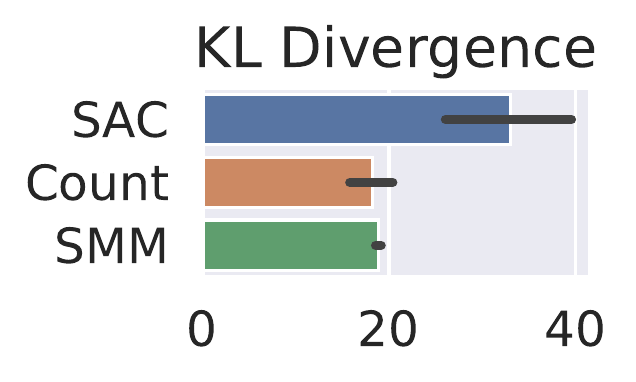}
    \includegraphics[width=0.49\columnwidth]{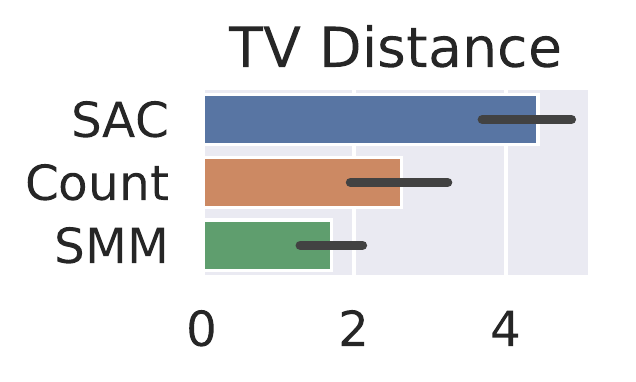}
    \caption{\textbf{Non-Uniform Exploration}:\;~We measure the discrepancy between the state marginal distribution, $\rho_\pi(s)$, and a non-uniform target distribution. SMM matches the target distribution better than SAC and is on par with Count. Error bars show std.\ dev.\ across 4 random seeds. \label{fig:prior-experiment}}
\end{figure}

\subsection{Computational Complexity}
We compare the wall-clock time of each exploration method in Table~\ref{table:wall-clock-time}. The computational cost of our method is comparable with prior work.

\begin{table}[t]
\centering
\caption{Average wall-clock time per epoch (1e3 env steps) on \emph{Fetch}.\label{table:wall-clock-time}}
\begin{tabular}{c|c|c}
Method & Wall-clock time (s) & \% Difference \\
\hline
SAC & 17.95s & +0\% \\
ICM & 22.74s & +27\% \\ 
Count & 25.24s & +41\% \\
\textbf{SMM (ours)} & \textbf{25.82s} & \textbf{+44\%}\\
PseudoCounts & 33.87s & +89\%
\end{tabular}
\end{table}

\subsection{SMM Ablation Study}

To understand the relative contribution of each component in the SM4 objective (Eq.~\ref{eq:smm-mop-objective}), we compare SM4 to baselines that lack conditional state entropy $\mathcal{H}_{\pi_z}[s] = - \log \rho_{\pi_z}(s)$, latent conditional action entropy $\log p(z \mid s)$, or both (i.e, SAC). In Fig.~\ref{fig:point-3}, we plot the training time performance on the \emph{Navigation} task with 3 halls of length 50. We see that SM4 relies heavily on both key differences from SAC.

In Fig.~\ref{fig:fetch-uniform-train}, we also plot the latent action entropy $\mathcal{H}[z \mid s]$ (discriminator) and latent state entropy $\mathcal{H}[s \mid z]$ (density model) per epoch for SM4 with varying number of mixture components.

\begin{figure}[t]
\includegraphics[width=\columnwidth]{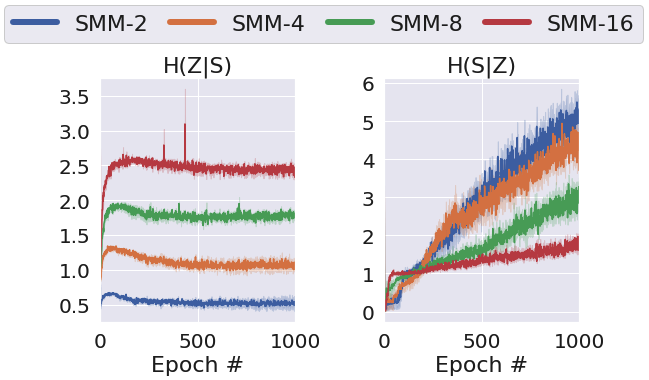}
\caption{The latent action entropy $\mathcal{H}[z \mid s]$ (discriminator) and latent state entropy $\mathcal{H}[s \mid z]$ (density model) per epoch.\label{fig:fetch-uniform-train}}
\end{figure}

\subsection{Visualizing Mixture Components of SM4}

\begin{figure*}[h]
    \centering
    \includegraphics[width=\textwidth]{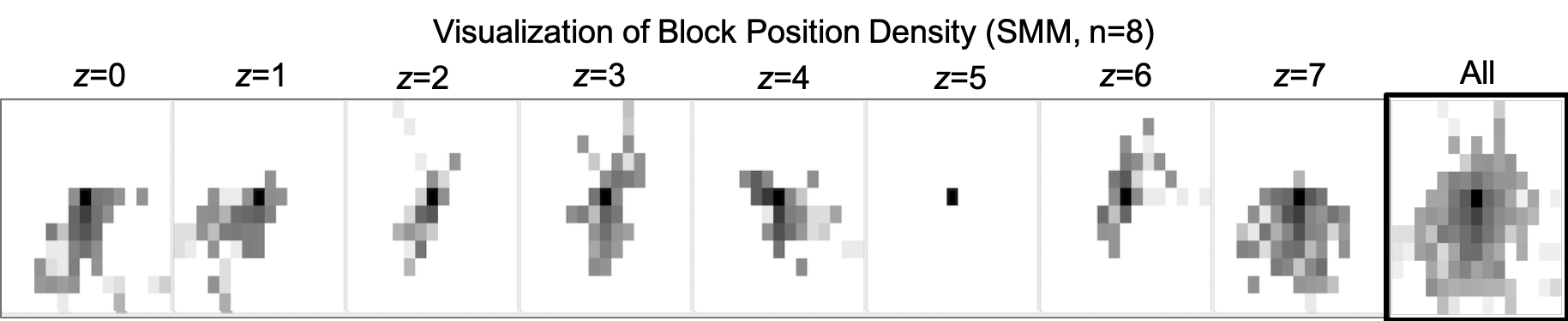}
    \caption{\textbf{SM4 with Eight Mixture Components}. In \emph{Fetch}, we plot the log state marginal $\log \rho_{\pi_z}(s)$ over block XY-coordinates for each latent component $z \in \{0, \ldots, 7\}$, results are averaged over 1000 epochs.}
    \label{fig:visualization-smm-8}
\end{figure*}

In Fig.~\ref{fig:visualization-smm-8}, we visualize the state marginals of each mixture component of SM4 for the \emph{Fetch} task. The policy was trained using a uniform target distribution.

\section{Implementation Details}
\label{appendix:implementation-details}

\subsection{Environment Details}\label{appendix:environment-details}

We summarize the environment parameters for \emph{Navigation}, \emph{Fetch}, and \emph{D'Claw} in Table~\ref{table:environment}.

\textbf{{Fetch}}. We used the simulated Fetch Robotics arm\footnote{\url{https://fetchrobotics.com/}} implemented by~\citet{plappert2018multi} using the MuJoCo simulator~\cite{todorov2012mujoco}. The state vector $s \in \mathbb{R}^{28}$ includes the xyz-coordinates $s_\text{obj}, s_\text{robot} \in \mathbb{R}^3$ of the block and the robot gripper respectively, as well as their velocities, orientations, and relative position $s_\text{obj} - s_\text{robot}$. At the beginning of each episode, we spawn the object at the center of the table, and the robot gripper above the initial block position. We terminate each episode after 50 environment steps, or if the block falls off the table.

We considered two target state marginal distributions. In \emph{Fetch-Uniform}, the target density is given by
$$
p^*(s) \propto \exp\left( \alpha_1 r_\text{goal}(s) + \alpha_2 r_\text{robot}(s) + \alpha_3 r_\text{action}(s) \right)
$$
where $\alpha_1, \alpha_2, \alpha_3 > 0$ are fixed weights, and the rewards
\begin{align*}
    r_\text{goal}(s) &:= 1 - \mathbbm{1}( \text{$s_\text{obj}$ is on the table surface}) \\
    r_\text{robot}(s) &:= \mathbbm{1}( \|s_\text{obj} - s_\text{robot} \|_2^2 < 0.1)\\
    r_\text{action}(s) &:= - \| a \|_2^2
\end{align*}
correspond to (1) a uniform distribution of the block position over the table surface (the agent receives +0 reward while the block is on the table), (2) an indicator reward for moving the robot gripper close to the block, and (3) action penalty, respectively. %
The environment reward is a weighted sum of the three reward terms: $r_\text{env}(s) \triangleq 20 r_\text{goal}(s) + r_\text{robot}(s) + 0.1 r_\text{action}(s)$. At test-time, we sample a goal block location $g \in \mathbb{R}^3$ uniformly on the table surface, and the goal is not observed by the agent.%

In \emph{Fetch-Half}, the target state density places higher probability mass to states where the block is on the left-side of the table. This is implemented by replacing $r_\text{goal}(s)$ with a reward function that gives a slightly higher reward +0.1 for states where the block is on the left-side of the table.

\textbf{{D'Claw}}. The \emph{D'Claw} robot~\citep{ahn2019robel}\footnote{\url{www.roboticsbenchmarks.org}}
controls three claws to rotate a valve object. The environment consists of a 9-dimensional action space (three joints per claw) and a 12-dimensional observation space that encodes the joint angles and object orientation. We fixed each episode at 50 timesteps, which is about 5 seconds on the real robot. In the hardware experiments, each algorithm was trained on the same four \emph{D'Claw} robots to ensure consistency.

We defined the target state distribution to place uniform probability mass over all object angles in $[-180^\circ, 180^\circ]$. It also incorporates reward shaping terms that place lower probability mass on states with high joint velocity and on states with joint positions that deviate far from the initial position (see~\citep{zhu2019dexterous}).

\textbf{Navigation}: Episodes have a maximum time horizon of 100 steps. The environment reward is
\begin{equation*}
r_\text{env}(s) = \begin{cases}
p_i & \text{if }\|s_\text{robot} - g_i\|_2^2 < \epsilon\text{ for any }i \in [n] \\
0 & \text{otherwise}
\end{cases}
\end{equation*}
where $s_{xy}$ is the xy-position of the agent.
We used a uniform target distribution over the end of all $m$ halls, so the environment reward at training time is $r_\text{env}(s) = \frac{1}{m}$ if the robot is close enough to the end of any of the halls.

We used a fixed hall length of 10 in Figures~\ref{fig:point-2d-density} and~\ref{fig:point-2d-results}, and length 50 in Fig.~\ref{fig:point-3}. All experiments used $m=3$ halls, except in Fig.~\ref{fig:point-2d-results} where we varied the number of halls $\{3, 5, 7\}$.

\begin{table*}[t]
\centering
\caption{\textbf{Environment parameters} specifying the observation space dimension $|\mathcal{S}|$; action space dimension $|\mathcal{A}|$; max episode length $T$; the environment reward, related to the target distribution by $\exp\{ r_\text{env}(s) \} \propto p^*(s)$, and other environment parameters.\label{table:environment}}
\begin{tabular}{|l|l|l|l|p{3cm}|l|l|} 
\hline
Environment                                & $|\mathcal{S}|$ & $|\mathcal{A}|$ &  $T$                           & Env Reward ($\log p^*(s))$                    & Other Parameters                                                                   & Figures  \\ 
\hline
\multirow{3}{*}{\textit{Navigation} } & \multirow{2}{*}{2}                                & \multirow{2}{*}{2~}                                  & \multirow{2}{*}{100}                                            & Uniform over all $m$ halls                      & \begin{tabular}[c]{@{}l@{}}\# Halls: 3, 5, 7\\Hall length:~ 10\end{tabular}  & \ref{fig:point-2d-density}, \ref{fig:point-2d-results}       \\ 
\cline{5-7}
                                      &                                                   &                                                      &                                                                                       & Uniform over all $m$ halls                      & \begin{tabular}[c]{@{}l@{}}\# Halls: 3\\Hall length: 50\end{tabular}         & \ref{fig:point-3}      \\ 
\hline
\multirow{2}{*}{\textit{Fetch} }      & \multirow{2}{*}{25}                               & \multirow{2}{*}{4}                                   & \multirow{2}{*}{50}             & Uniform block pos. over table surface         &                                                                              & \ref{fig:simulated-manipulation-state-entropy}, \ref{fig:visualization-algos}, \ref{fig:fetch-test}, \ref{fig:hist-avg}, \ref{fig:visualization-smm-8}, \ref{fig:fetch-uniform-train}  \\ 
\cline{5-7}
                                      &                                                   &                                                      &                                                                   & More block pos. density on left-half of table &                                                                              & \ref{fig:prior-experiment}       \\ 
\hline
\textit{D'Claw} & 12 & 9 & 50 & Uniform object angle over $[-180^\circ, 180^\circ]$ & & \ref{fig:dclaw-min-max-angle}, \ref{fig:dclaw-angle} \\ 
\hline
\end{tabular}
\end{table*}{}

\subsection{GAIL ablation}\label{section:gail}

GAIL assumes access to expert demonstrations, which SMM and the other exploration methods do not require. To compare GAIL with the exploration methods on a level footing, we sampled synthetic states from $p^*(s)$ to train GAIL, and restricted the GAIL discriminator input to states only (no actions).

For \emph{D'Claw}, we sampled the valve object angle uniformly in $[-180^\circ, 180^\circ]$. For \emph{Fetch-Uniform}, we sampled object positions $s_{\text{object}}$ uniformly on the table surface, and tried two different sampling distributions for the gripper position $s_{\text{robot}}$ (see Fig.~\ref{fig:gail-ablation}). For both environments, all other state dimensions were sampled uniformly in $[-10, 10]$, and used 1e4 synthetic state samples to train GAIL.

Since the state samples from $p^*(s)$ may not be reachable from the initial state, the policy may not be able to fool the discriminator. To get around this problem, we also tried training GAIL with the discriminator input restricted to only the state dimensions corresponding to the object position or gripper position (\emph{Fetch}), or the object angle (\emph{D'Claw}). We summarize these GAIL ablation experiments in Fig.~\ref{fig:gail-ablation}. In our experiments, we used the best GAIL ablation model to compare against the exploration baselines. %

\begin{figure*}[t]
    \centering
    \begin{subfigure}[b]{\columnwidth}
        \centering
        \includegraphics[align=t,width=0.9\textwidth]{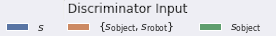}
        \includegraphics[align=t,width=0.4\textwidth]{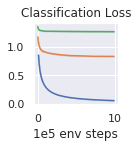}
        \includegraphics[align=t,width=0.59\textwidth,trim=0 0 130pt 0,clip]{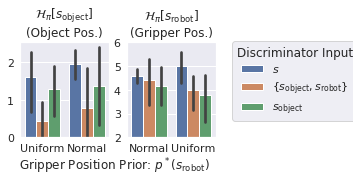}
        \caption{GAIL ablations on \emph{Fetch}\label{fig:gail-ablation-manipulation}}
    \end{subfigure}%
    \hspace{10pt}
    \begin{subfigure}[b]{\columnwidth}
        \centering
        \includegraphics[align=t,width=0.9\textwidth]{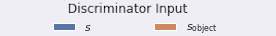}
        \includegraphics[align=t,width=0.4\textwidth]{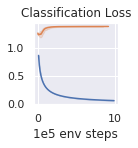}
        \includegraphics[align=t,width=0.52\textwidth,trim=0 0 130pt 0,clip]{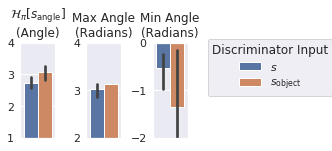}
        \vspace{12pt}
        \caption{GAIL ablations on \emph{D'Claw}\label{fig:gail-ablation-dclaw}}
    \end{subfigure}%
    \caption{\textbf{GAIL ablation study}: We studied the effect of restricting the GAIL discriminator input to fewer state dimensions. \textbf{(a) \emph{Fetch}}: We trained the GAIL discriminator on the entire state vector $s$; on the object and gripper positions $\{s_\text{object}, s_\text{robot}\}$ only; or on the object position $s_\text{object}$ only.
    We also varied the sampling distribution for the gripper position, $p^*(s_\text{robot})$: we compare using a normal distribution, $\mathcal{N}(s_{\text{object}}, I_3)$, to sample gripper positions closer to the object, versus a uniform distribution, $\text{Uniform}[-10, 10]$, for greater entropy of the sampled gripper positions. We observe that sampling gripper positions closer to the object position improves the entropy of the object position $\mathcal{H}_\pi[s_\text{object}]$, but hurts the entropy of the gripper position  $\mathcal{H}_\pi[s_\text{robot}]$.
    \textbf{(b) \emph{D'Claw}}: We restricted the discriminator to the entire state vector $s$, or to the object angle and position $s_\text{object}$. \textbf{Analysis}: In both domains, we observe that restricting the discriminator input to fewer state dimensions (e.g., to $s_\text{object}$) makes the discriminator less capable of distinguishing between expert and policy states (orange and green curves). On the other hand, training on the entire state vector $s$ causes the discriminator loss to approach 0 (i.e., perfect classification), partly because some of the ``expert'' states sampled from $p^*(s)$ are not reachable from the initial state, and the policy is thus unable to fool the discriminator.
    \label{fig:gail-ablation}}
\end{figure*}

\subsection{VAE Density Model}
In our SMM implementation, we estimated the density of data $x$ as $p(x) \approx \text{decoder}(\hat{x} = x | z=\text{encoder}(x))$. That is, we encoded $x$ to $z$, reconstruction $\hat{x}$ from $z$, and then took the likelihood of the true data $x$ under a unit-variance Gaussian distribution centered at the reconstructed $\hat{x}$. The log-likelihood is therefore given by the mean-squared error between the data $x$ and the reconstruction $\hat{x}$, plus a constant that is independent of $x$: $\log q(x) = \frac{1}{2}\|x - \hat{x}\|_2^2 + C$.

\subsection{Algorithm Hyperparameters}

We summarize hyperparameter settings in Table~\ref{table:hyperparameters}. All algorithms were trained for 1e6 steps on \emph{Fetch}, 1e6 steps on \emph{D'Claw} Sim2Real, 1e5 steps on \emph{D'Claw} hardware, and 1e5 steps on \emph{Navigation}.

\textbf{Loss Hyperparameters}. For each exploration method, we tuned the weights of the different loss components. %
\emph{SAC reward scale} controls the weight of the action entropy reward relative to the extrinsic reward. \emph{Count coeff} controls the intrinsic count-based exploration reward w.r.t.\ the extrinsic reward and SAC action entropy reward. Similarly, \emph{Pseudocount coeff} controls the intrinsic pseudocount exploration reward. \emph{SMM coeff for $\mathcal{H}[s \mid z]$ and $\mathcal{H}[z \mid s]$} control the weight of the different loss components (state entropy and latent conditional entropy) of the SMM objective in Eq.~\ref{eq:smm-mop-objective}.

\textbf{Historical Averaging}. In the \emph{Fetch} experiments, we tried the following sampling strategies for historical averaging: (1) \emph{Uniform}: Sample policies uniformly across training iterations. (2) \emph{Exponential}: Sample policies, with recent policies sampled exponentially more than earlier ones. (3) \emph{Last}: Sample the $N$ latest policies uniformly at random. We found that \emph{Uniform} worked less well, possibly due to the policies at early iterations not being trained enough. We found negligible difference in the state entropy metric between \emph{Exponential} vs. \emph{Last}, and between sampling 5 vs. 10 historical policies, and we also note that it is unnecessary to keep checkpoints from every iteration.

\textbf{Network Hyperparameters}. For all algorithms, we use a Gaussian policy with two hidden layers with Tanh activation and a final fully-connected layer. The Value function and Q-function each are a feedforward MLP with two hidden layers with ReLU activation and a final fully-connected layer. Each hidden layer is of size 300 (SMM, SAC, ICM, C, PC) or 256 (GAIL). The same network configuration is used for the SMM discriminator, $d(z \mid s)$, and the GAIL discriminator, but with different input and output sizes. The SMM density model, $q(s)$, is modeled by a VAE with encoder and decoder networks each consisting of two hidden layers of size (150, 150) with ReLU activation. The same VAE network configuration is used for Pseudocount.

\textbf{GAIL Hyperparameters}: The replay buffer is filled with 1e4 random actions before training, for training stability. We perform one discriminator update per SAC update. For both \emph{Fetch} and \emph{D'Claw}, we used 1e4 states sampled from $p^*(s)$. Other hyperparameter settings, such as batch size for both discriminator and policy updates, are summarized in Table~\ref{table:hyperparameters}. We observed that GAIL training is more unstable compared to the exploration baselines. Thus, for GAIL, we did not take the final iterate (e.g., policy at convergence) but instead used early termination (e.g., take the best iterate according to the state entropy metric). %

\begin{table*}[t]\centering\footnotesize
\caption{\footnotesize\textbf{Hyperparameter settings}. Hyperparameters were chosen according to the following eval metrics: \emph{Fetch-Uniform}: State entropy of the discretized gripper and block positions (bin size 0.05), after rolling out the trained policy for 50K env steps. \emph{Fetch-Half}: $\kl{p^*(s)}{\rho_\pi(s)}$ and $\text{TV}(p^*(s), \rho_\pi(s))$ of the discretized gripper and block positions (bin size 0.01), after rolling out the trained policy for 50K env steps. \emph{2D Navigation}: State entropy of the discretized XY-positions of the trained policy. \emph{D'Claw}: State entropy of the object angle. \label{table:hyperparameters}}
\begin{tabular}{|l|l|l|l|}
\hline
\textbf{Environment} & \textbf{Algorithm} & \textbf{Hyperparameters Used}  & \textbf{Hyperparameters Considered}   \\ \hline
\multirow{2}{*}{All}    & \begin{tabular}[c]{@{}l@{}}SMM, SAC,\\ICM, Count,\\Pseudocount\end{tabular}       & \begin{tabular}[c]{@{}l@{}}Batch size: 128\\ 1e6 env training steps\\ RL discount: 0.99\\ Network size: 300\\ Policy lr: 3e-4\\ Q-function lr: 3e-4\\ Value function lr: 3e-4\end{tabular} & N/A (Default SAC hyperparameters)    \\
\cline{2-4}
& GAIL & \begin{tabular}[c]{@{}l@{}}
1e6 env training steps\\
Policy lr: 1e-5\\
Critic lr: 1e-3\\
\# Random actions\\
\phantom{xx}before training: 1e4\\
Network size: 256
\end{tabular}
& N/A (Default GAIL hyperparameters)\\
\hline %
\multirow{2}{*}{\begin{tabular}[c]{@{}l@{}}\emph{Navigation}\end{tabular}}                  & SMM, SAC       & SAC reward scale: 25                                                                                                                                                                               & SAC reward scale: 1e-2, 0.1, 1, 10, 25, 100                                                                                                                                                \\ \cline{2-4}
& SMM & \begin{tabular}[c]{@{}l@{}}SMM $\mathcal{H}[s \mid z]$ coeff: 1\\SMM $\mathcal{H}[z \mid s]$ coeff: 1\\\end{tabular} & \begin{tabular}[c]{@{}l@{}}SMM $\mathcal{H}[s \mid z]$ coeff: 1e-3, 1e-2, 1e-1, 1, 10\\SMM $\mathcal{H}[z \mid s]$ coeff: 1e-3, 1e-2, 1e-1, 1, 10\\\end{tabular} \\ \hline
\multirow{6}{*}{\begin{tabular}[c]{@{}l@{}}\emph{Fetch-Uniform}\end{tabular}} & SMM    & \begin{tabular}[c]{@{}l@{}}Num skills: 4\\ VAE lr: 1e-2\\SMM $\mathcal{H}[s \mid z]$ coeff: 1\\SMM $\mathcal{H}[z \mid s]$ coeff: 1\\ HA sampling: Exponential\\ \# HA policies: 10\\ SMM Latent Prior Coeff: 1\end{tabular}  & \begin{tabular}[c]{@{}l@{}}Num skills: 1, 2, 4, 8, 16\\ VAE lr: 1e-4, 1e-3, 1e-2\\\\\\ HA sampling: Exponential, Uniform, Last\\ \# HA policies: 5, 10\\ SMM Latent Prior Coeff: 1, 4\end{tabular}  \\ \cline{2-4}
                               & SAC       & SAC reward scale: 0.1                                                                                                                                                                              & SAC reward scale: 0.1, 1, 10, 100                                                                                                                                                          \\ \cline{2-4} 
                               & Count     & \begin{tabular}[c]{@{}l@{}}Count coeff: 10\\ Histogram bin width: 0.05\end{tabular}                                                                                                                & \begin{tabular}[c]{@{}l@{}}Count coeff: 0.1, 1, 10\\\end{tabular}                                                                                                                                                                    \\ \cline{2-4} 
                               & Pseudocount        & \begin{tabular}[c]{@{}l@{}}Pseudocount coeff: 1\\ VAE lr: 1e-2\end{tabular}                                                                                                                                 & \begin{tabular}[c]{@{}l@{}}Pseudocount coeff: 0.1, 1, 10\\ (Use same VAE lr as SMM)\end{tabular}                                                                                                    \\ \cline{2-4} 
                               & ICM       & Learning rate: 1e-3                                                                                                                                                                                & Learning rate: 1e-4, 1e-3, 1e-2                                                                                                                                                            \\ \cline{2-4} 
& GAIL       &\begin{tabular}[c]{@{}l@{}}
Batch size: 512\\
\# SAC updates per step: 1\\
Discriminator input: $s$\\
Training iterate: 1e6\\
\# State Samples: 1e4
\end{tabular}                                                                                                                                                                                & \begin{tabular}[c]{@{}l@{}}
Batch size: 128, 512, 1024\\
\# SAC updates per step: 1, 4\\
Discriminator input: $s$, $s_\text{object}$, $\{s_\text{object}, s_\text{robot}\}$\\
Training iterate: 1e5, 2e5, 3e5, $\ldots$, 9e5, 1e6\\
\# State Samples: 1e4
\end{tabular}\\ \hline
\multirow{4}{*}{\begin{tabular}[c]{@{}l@{}}\emph{Fetch-Half}\end{tabular}}    & \begin{tabular}[c]{@{}l@{}}SMM, SAC,\\ICM, Count\end{tabular}       & SAC reward scale: 0.1                                                                                                                                                                              & (Best reward scale for \emph{Fetch-Uniform})                                                                                                                                                      \\ \cline{2-4} 
                               & SMM       & \begin{tabular}[c]{@{}l@{}}Num skills: 4\\SMM $\mathcal{H}[s \mid z]$ coeff: 1\\SMM $\mathcal{H}[z \mid s]$ coeff: 1\\\end{tabular}                                                                                                                                                                                   & \begin{tabular}[c]{@{}l@{}}Num skills: 1, 2, 4, 8\\\\\end{tabular}                                                                                                                                                                   \\ \cline{2-4} 
                               & Count     & \begin{tabular}[c]{@{}l@{}}Count coeff: 10\\ Histogram bin width: 0.05\end{tabular}                                                                                                                & \begin{tabular}[c]{@{}l@{}}Count coeff: 0.1, 1, 10\\\end{tabular}                                                                                                                                                                   \\ \cline{2-4} 
                               & ICM       & Learning rate: 1e-3                                                                                                                                                                                & Learning rate: 1e-4, 1e-3, 1e-2                                                                                                                                                            \\ \hline
\multirow{6}{*}{\begin{tabular}[c]{@{}l@{}}\emph{D'Claw}\end{tabular}}                  & \begin{tabular}[c]{@{}l@{}}SMM, SAC\end{tabular}       & SAC reward scale: 5                                                                                                                                                                               & SAC reward scale: 1e-2, 0.1, 1, 5, 10, 100                                                                                                                                                \\ \cline{2-4}
& SMM & \begin{tabular}[c]{@{}l@{}}SMM $\mathcal{H}[s \mid z]$ coeff: 250\end{tabular} & SMM $\mathcal{H}[s \mid z]$ coeff: 1, 10, 100, 250, 500, 1e3 \\

\cline{2-4} 
& Count       &\begin{tabular}[c]{@{}l@{}}
Count coeff: 1\\
Histogram bin width: 0.05
\end{tabular}                                                            & \begin{tabular}[c]{@{}l@{}}
Count coeff: 1, 10\\
Histogram bin width: 0.05, 0.1
\end{tabular}\\

\cline{2-4} 
& Pseudocount       &\begin{tabular}[c]{@{}l@{}}
Pseudocount coeff: 1\\
VAE lr: 1e-3
\end{tabular}                                                             & \begin{tabular}[c]{@{}l@{}}
Pseudocount coeff: 1, 10\\
VAE lr: 1e-1, 1e-2, 1e-3
\end{tabular}\\

\cline{2-4} 
& ICM       &\begin{tabular}[c]{@{}l@{}}
Learning rate: 1e-3\\
VAE lr: 1e-1
\end{tabular}                                                             & \begin{tabular}[c]{@{}l@{}}
Learning rate: 1e-2, 1e-3, 1e-4\\
VAE lr: 1e-1, 1e-2, 1e-3
\end{tabular}\\

\cline{2-4} 
& GAIL       &\begin{tabular}[c]{@{}l@{}}
Batch size: 512\\
\# SAC updates per step: 4\\
Discriminator input: $s_\text{object}$\\
Training iterate: 1e5\\
\# State Samples: 1e4
\end{tabular}                                                             & \begin{tabular}[c]{@{}l@{}}
Batch size: 128, 512, 1024\\
\# SAC updates per step: 1, 4\\
Discriminator input: $s$, $s_\text{object}$\\
Training iterate: 1e5, 2e5, 3e5, $\ldots$, 9e5, 1e6\\
\# State Samples: 1e4
\end{tabular}\\ \hline
\end{tabular}
\end{table*}{}


\begin{thebibliography}{59}
\providecommand{\natexlab}[1]{#1}
\providecommand{\url}[1]{\texttt{#1}}
\expandafter\ifx\csname urlstyle\endcsname\relax
  \providecommand{\doi}[1]{doi: #1}\else
  \providecommand{\doi}{doi: \begingroup \urlstyle{rm}\Url}\fi

\bibitem[Abdolmaleki et~al.(2018)Abdolmaleki, Springenberg, Tassa, Munos,
  Heess, and Riedmiller]{abdolmaleki2018maximum}
Abdolmaleki, A., Springenberg, J.~T., Tassa, Y., Munos, R., Heess, N., and
  Riedmiller, M.
\newblock Maximum a posteriori policy optimisation.
\newblock \emph{arXiv preprint arXiv:1806.06920}, 2018.

\bibitem[Achiam et~al.(2018)Achiam, Edwards, Amodei, and
  Abbeel]{achiam2018variational}
Achiam, J., Edwards, H., Amodei, D., and Abbeel, P.
\newblock Variational option discovery algorithms.
\newblock \emph{arXiv preprint arXiv:1807.10299}, 2018.

\bibitem[Agakov(2004)]{agakov2004algorithm}
Agakov, D. B.~F.
\newblock The im algorithm: a variational approach to information maximization.
\newblock \emph{Advances in Neural Information Processing Systems},
  16:\penalty0 201, 2004.

\bibitem[Ahn et~al.(2019)Ahn, Zhu, Hartikainen, Ponte, Gupta, Levine, and
  Kumar]{ahn2019robel}
Ahn, M., Zhu, H., Hartikainen, K., Ponte, H., Gupta, A., Levine, S., and Kumar,
  V.
\newblock Robel: Robotics benchmarks for learning with low-cost robots.
\newblock \emph{arXiv preprint arXiv:1909.11639}, 2019.

\bibitem[Bellemare et~al.(2016)Bellemare, Srinivasan, Ostrovski, Schaul,
  Saxton, and Munos]{bellemare2016unifying}
Bellemare, M., Srinivasan, S., Ostrovski, G., Schaul, T., Saxton, D., and
  Munos, R.
\newblock Unifying count-based exploration and intrinsic motivation.
\newblock In \emph{Advances in Neural Information Processing Systems}, pp.\
  1471--1479, 2016.

\bibitem[Brown(1951)]{brown1951iterative}
Brown, G.
\newblock Iterative solution of games by fictitious play.
\newblock \emph{Activity Analysis of Production and Allocation}, 1951.

\bibitem[Burda et~al.(2018)Burda, Edwards, Storkey, and
  Klimov]{burda2018exploration}
Burda, Y., Edwards, H., Storkey, A., and Klimov, O.
\newblock Exploration by random network distillation.
\newblock \emph{arXiv preprint arXiv:1810.12894}, 2018.

\bibitem[Chentanez et~al.(2005)Chentanez, Barto, and
  Singh]{chentanez2005intrinsically}
Chentanez, N., Barto, A.~G., and Singh, S.~P.
\newblock Intrinsically motivated reinforcement learning.
\newblock In \emph{Advances in neural information processing systems}, pp.\
  1281--1288, 2005.

\bibitem[Christiano et~al.(2017)Christiano, Leike, Brown, Martic, Legg, and
  Amodei]{christiano2017deep}
Christiano, P.~F., Leike, J., Brown, T., Martic, M., Legg, S., and Amodei, D.
\newblock Deep reinforcement learning from human preferences.
\newblock In \emph{Advances in Neural Information Processing Systems}, pp.\
  4299--4307, 2017.

\bibitem[Co-Reyes et~al.(2018)Co-Reyes, Liu, Gupta, Eysenbach, Abbeel, and
  Levine]{co2018self}
Co-Reyes, J.~D., Liu, Y., Gupta, A., Eysenbach, B., Abbeel, P., and Levine, S.
\newblock Self-consistent trajectory autoencoder: Hierarchical reinforcement
  learning with trajectory embeddings.
\newblock \emph{arXiv preprint arXiv:1806.02813}, 2018.

\bibitem[Colas et~al.(2018)Colas, Sigaud, and Oudeyer]{colas2018curious}
Colas, C., Sigaud, O., and Oudeyer, P.-Y.
\newblock Curious: Intrinsically motivated multi-task, multi-goal reinforcement
  learning.
\newblock \emph{arXiv preprint arXiv:1810.06284}, 2018.

\bibitem[Daskalakis \& Pan(2014)Daskalakis and Pan]{daskalakis2014counter}
Daskalakis, C. and Pan, Q.
\newblock A counter-example to karlin's strong conjecture for fictitious play.
\newblock In \emph{2014 IEEE 55th Annual Symposium on Foundations of Computer
  Science}, pp.\  11--20. IEEE, 2014.

\bibitem[Duan et~al.(2016)Duan, Schulman, Chen, Bartlett, Sutskever, and
  Abbeel]{duan2016rl}
Duan, Y., Schulman, J., Chen, X., Bartlett, P.~L., Sutskever, I., and Abbeel,
  P.
\newblock Rl\^2: Fast reinforcement learning via slow reinforcement learning.
\newblock \emph{arXiv preprint arXiv:1611.02779}, 2016.

\bibitem[Eysenbach et~al.(2018)Eysenbach, Gupta, Ibarz, and
  Levine]{eysenbach2018diversity}
Eysenbach, B., Gupta, A., Ibarz, J., and Levine, S.
\newblock Diversity is all you need: Learning skills without a reward function.
\newblock \emph{arXiv preprint arXiv:1802.06070}, 2018.

\bibitem[Finn et~al.(2016)Finn, Levine, and Abbeel]{finn2016guided}
Finn, C., Levine, S., and Abbeel, P.
\newblock Guided cost learning: Deep inverse optimal control via policy
  optimization.
\newblock In \emph{International Conference on Machine Learning}, pp.\  49--58,
  2016.

\bibitem[Finn et~al.(2017)Finn, Abbeel, and Levine]{finn2017model}
Finn, C., Abbeel, P., and Levine, S.
\newblock Model-agnostic meta-learning for fast adaptation of deep networks.
\newblock In \emph{Proceedings of the 34th International Conference on Machine
  Learning-Volume 70}, pp.\  1126--1135. JMLR. org, 2017.

\bibitem[Fortunato et~al.(2017)Fortunato, Azar, Piot, Menick, Osband, Graves,
  Mnih, Munos, Hassabis, Pietquin, et~al.]{fortunato2017noisy}
Fortunato, M., Azar, M.~G., Piot, B., Menick, J., Osband, I., Graves, A., Mnih,
  V., Munos, R., Hassabis, D., Pietquin, O., et~al.
\newblock Noisy networks for exploration.
\newblock \emph{arXiv preprint arXiv:1706.10295}, 2017.

\bibitem[Fu et~al.(2017)Fu, Luo, and Levine]{fu2017learning}
Fu, J., Luo, K., and Levine, S.
\newblock Learning robust rewards with adversarial inverse reinforcement
  learning.
\newblock \emph{arXiv preprint arXiv:1710.11248}, 2017.

\bibitem[Fu et~al.(2018)Fu, Luo, and Levine]{fu2018learning}
Fu, J., Luo, K., and Levine, S.
\newblock Learning robust rewards with adverserial inverse reinforcement
  learning.
\newblock In \emph{International Conference on Learning Representations}, 2018.
\newblock URL \url{https://openreview.net/forum?id=rkHywl-A-}.

\bibitem[Gupta et~al.(2018)Gupta, Mendonca, Liu, Abbeel, and
  Levine]{gupta2018meta}
Gupta, A., Mendonca, R., Liu, Y., Abbeel, P., and Levine, S.
\newblock Meta-reinforcement learning of structured exploration strategies.
\newblock In \emph{Advances in Neural Information Processing Systems}, pp.\
  5302--5311, 2018.

\bibitem[Haarnoja et~al.(2018)Haarnoja, Zhou, Abbeel, and
  Levine]{haarnoja2018soft}
Haarnoja, T., Zhou, A., Abbeel, P., and Levine, S.
\newblock Soft actor-critic: Off-policy maximum entropy deep reinforcement
  learning with a stochastic actor.
\newblock \emph{arXiv preprint arXiv:1801.01290}, 2018.

\bibitem[Hazan et~al.(2018)Hazan, Kakade, Singh, and
  Van~Soest]{hazan2018provably}
Hazan, E., Kakade, S.~M., Singh, K., and Van~Soest, A.
\newblock Provably efficient maximum entropy exploration.
\newblock \emph{arXiv preprint arXiv:1812.02690}, 2018.

\bibitem[Held et~al.(2017)Held, Geng, Florensa, and Abbeel]{held2017automatic}
Held, D., Geng, X., Florensa, C., and Abbeel, P.
\newblock Automatic goal generation for reinforcement learning agents.
\newblock \emph{arXiv preprint arXiv:1705.06366}, 2017.

\bibitem[Ho \& Ermon(2016)Ho and Ermon]{ho2016generative}
Ho, J. and Ermon, S.
\newblock Generative adversarial imitation learning.
\newblock In \emph{Advances in Neural Information Processing Systems}, pp.\
  4565--4573, 2016.

\bibitem[Houthooft et~al.(2016)Houthooft, Chen, Duan, Schulman, De~Turck, and
  Abbeel]{houthooft2016vime}
Houthooft, R., Chen, X., Duan, Y., Schulman, J., De~Turck, F., and Abbeel, P.
\newblock Vime: Variational information maximizing exploration.
\newblock In \emph{Advances in Neural Information Processing Systems}, pp.\
  1109--1117, 2016.

\bibitem[Kaelbling(1993)]{kaelbling1993learning}
Kaelbling, L.~P.
\newblock Learning to achieve goals.
\newblock In \emph{IJCAI}, pp.\  1094--1099. Citeseer, 1993.

\bibitem[Kappen et~al.(2012)Kappen, G{\'o}mez, and Opper]{kappen2012optimal}
Kappen, H.~J., G{\'o}mez, V., and Opper, M.
\newblock Optimal control as a graphical model inference problem.
\newblock \emph{Machine learning}, 87\penalty0 (2):\penalty0 159--182, 2012.

\bibitem[Kolter \& Ng(2009)Kolter and Ng]{kolter2009near}
Kolter, J.~Z. and Ng, A.~Y.
\newblock Near-bayesian exploration in polynomial time.
\newblock In \emph{Proceedings of the 26th annual international conference on
  machine learning}, pp.\  513--520, 2009.

\bibitem[Levine(2018)]{levine2018reinforcement}
Levine, S.
\newblock Reinforcement learning and control as probabilistic inference:
  Tutorial and review.
\newblock \emph{arXiv preprint arXiv:1805.00909}, 2018.

\bibitem[Lillicrap et~al.(2015)Lillicrap, Hunt, Pritzel, Heess, Erez, Tassa,
  Silver, and Wierstra]{lillicrap2015continuous}
Lillicrap, T.~P., Hunt, J.~J., Pritzel, A., Heess, N., Erez, T., Tassa, Y.,
  Silver, D., and Wierstra, D.
\newblock Continuous control with deep reinforcement learning.
\newblock \emph{arXiv preprint arXiv:1509.02971}, 2015.

\bibitem[Mishra et~al.(2017)Mishra, Rohaninejad, Chen, and
  Abbeel]{mishra2017simple}
Mishra, N., Rohaninejad, M., Chen, X., and Abbeel, P.
\newblock A simple neural attentive meta-learner.
\newblock \emph{arXiv preprint arXiv:1707.03141}, 2017.

\bibitem[Nair et~al.(2018)Nair, Pong, Dalal, Bahl, Lin, and
  Levine]{nair2018visual}
Nair, A.~V., Pong, V., Dalal, M., Bahl, S., Lin, S., and Levine, S.
\newblock Visual reinforcement learning with imagined goals.
\newblock In \emph{Advances in Neural Information Processing Systems}, pp.\
  9191--9200, 2018.

\bibitem[Nash(1951)]{nash1951non}
Nash, J.
\newblock Non-cooperative games.
\newblock \emph{Annals of mathematics}, pp.\  286--295, 1951.

\bibitem[Oudeyer et~al.(2007)Oudeyer, Kaplan, and Hafner]{oudeyer2007intrinsic}
Oudeyer, P.-Y., Kaplan, F., and Hafner, V.~V.
\newblock Intrinsic motivation systems for autonomous mental development.
\newblock \emph{IEEE transactions on evolutionary computation}, 11\penalty0
  (2):\penalty0 265--286, 2007.

\bibitem[Pathak et~al.(2017)Pathak, Agrawal, Efros, and
  Darrell]{pathak2017curiosity}
Pathak, D., Agrawal, P., Efros, A.~A., and Darrell, T.
\newblock Curiosity-driven exploration by self-supervised prediction.
\newblock In \emph{Proceedings of the IEEE Conference on Computer Vision and
  Pattern Recognition Workshops}, pp.\  16--17, 2017.

\bibitem[Peters et~al.(2010)Peters, Mulling, and Altun]{peters2010relative}
Peters, J., Mulling, K., and Altun, Y.
\newblock Relative entropy policy search.
\newblock In \emph{Twenty-Fourth AAAI Conference on Artificial Intelligence},
  2010.

\bibitem[Plappert et~al.(2017)Plappert, Houthooft, Dhariwal, Sidor, Chen, Chen,
  Asfour, Abbeel, and Andrychowicz]{plappert2017parameter}
Plappert, M., Houthooft, R., Dhariwal, P., Sidor, S., Chen, R.~Y., Chen, X.,
  Asfour, T., Abbeel, P., and Andrychowicz, M.
\newblock Parameter space noise for exploration.
\newblock \emph{arXiv preprint arXiv:1706.01905}, 2017.

\bibitem[Plappert et~al.(2018)Plappert, Andrychowicz, Ray, McGrew, Baker,
  Powell, Schneider, Tobin, Chociej, Welinder, et~al.]{plappert2018multi}
Plappert, M., Andrychowicz, M., Ray, A., McGrew, B., Baker, B., Powell, G.,
  Schneider, J., Tobin, J., Chociej, M., Welinder, P., et~al.
\newblock Multi-goal reinforcement learning: Challenging robotics environments
  and request for research.
\newblock \emph{arXiv preprint arXiv:1802.09464}, 2018.

\bibitem[Pong et~al.(2019)Pong, Dalal, Lin, Nair, Bahl, and
  Levine]{pong2019skew}
Pong, V.~H., Dalal, M., Lin, S., Nair, A., Bahl, S., and Levine, S.
\newblock Skew-fit: State-covering self-supervised reinforcement learning.
\newblock \emph{arXiv preprint arXiv:1903.03698}, 2019.

\bibitem[Puterman(2014)]{puterman2014markov}
Puterman, M.~L.
\newblock \emph{Markov Decision Processes.: Discrete Stochastic Dynamic
  Programming}.
\newblock John Wiley \& Sons, 2014.

\bibitem[Rakelly et~al.(2019)Rakelly, Zhou, Quillen, Finn, and
  Levine]{rakelly2019efficient}
Rakelly, K., Zhou, A., Quillen, D., Finn, C., and Levine, S.
\newblock Efficient off-policy meta-reinforcement learning via probabilistic
  context variables.
\newblock \emph{arXiv preprint arXiv:1903.08254}, 2019.

\bibitem[Rawlik et~al.(2013)Rawlik, Toussaint, and
  Vijayakumar]{rawlik2013stochastic}
Rawlik, K., Toussaint, M., and Vijayakumar, S.
\newblock On stochastic optimal control and reinforcement learning by
  approximate inference.
\newblock In \emph{Twenty-Third International Joint Conference on Artificial
  Intelligence}, 2013.

\bibitem[Robinson(1951)]{robinson1951iterative}
Robinson, J.
\newblock An iterative method of solving a game.
\newblock \emph{Annals of mathematics}, pp.\  296--301, 1951.

\bibitem[Schaul et~al.(2015)Schaul, Horgan, Gregor, and
  Silver]{schaul2015universal}
Schaul, T., Horgan, D., Gregor, K., and Silver, D.
\newblock Universal value function approximators.
\newblock In \emph{International conference on machine learning}, pp.\
  1312--1320, 2015.

\bibitem[Schmidhuber(1991)]{schmidhuber1991possibility}
Schmidhuber, J.
\newblock A possibility for implementing curiosity and boredom in
  model-building neural controllers.
\newblock In \emph{Proc. of the international conference on simulation of
  adaptive behavior: From animals to animats}, pp.\  222--227, 1991.

\bibitem[Schmidhuber(2010)]{schmidhuber2010formal}
Schmidhuber, J.
\newblock Formal theory of creativity, fun, and intrinsic motivation
  (1990--2010).
\newblock \emph{IEEE Transactions on Autonomous Mental Development}, 2\penalty0
  (3):\penalty0 230--247, 2010.

\bibitem[Schulman et~al.(2015)Schulman, Levine, Abbeel, Jordan, and
  Moritz]{schulman2015trust}
Schulman, J., Levine, S., Abbeel, P., Jordan, M., and Moritz, P.
\newblock Trust region policy optimization.
\newblock In \emph{International conference on machine learning}, pp.\
  1889--1897, 2015.

\bibitem[Stadie et~al.(2015)Stadie, Levine, and
  Abbeel]{stadie2015incentivizing}
Stadie, B.~C., Levine, S., and Abbeel, P.
\newblock Incentivizing exploration in reinforcement learning with deep
  predictive models.
\newblock \emph{arXiv preprint arXiv:1507.00814}, 2015.

\bibitem[Ta{\"\i}ga et~al.(2019)Ta{\"\i}ga, Fedus, Machado, Courville, and
  Bellemare]{taiga2019benchmarking}
Ta{\"\i}ga, A.~A., Fedus, W., Machado, M.~C., Courville, A., and Bellemare,
  M.~G.
\newblock Benchmarking bonus-based exploration methods on the arcade learning
  environment.
\newblock \emph{arXiv preprint arXiv:1908.02388}, 2019.

\bibitem[Tang et~al.(2017)Tang, Houthooft, Foote, Stooke, Chen, Duan, Schulman,
  DeTurck, and Abbeel]{tang2017exploration}
Tang, H., Houthooft, R., Foote, D., Stooke, A., Chen, O.~X., Duan, Y.,
  Schulman, J., DeTurck, F., and Abbeel, P.
\newblock \# exploration: A study of count-based exploration for deep
  reinforcement learning.
\newblock In \emph{Advances in neural information processing systems}, pp.\
  2753--2762, 2017.

\bibitem[Theodorou \& Todorov(2012)Theodorou and
  Todorov]{theodorou2012relative}
Theodorou, E.~A. and Todorov, E.
\newblock Relative entropy and free energy dualities: Connections to path
  integral and kl control.
\newblock In \emph{2012 IEEE 51st IEEE Conference on Decision and Control
  (CDC)}, pp.\  1466--1473. IEEE, 2012.

\bibitem[Todorov et~al.(2012)Todorov, Erez, and Tassa]{todorov2012mujoco}
Todorov, E., Erez, T., and Tassa, Y.
\newblock Mujoco: A physics engine for model-based control.
\newblock In \emph{2012 IEEE/RSJ International Conference on Intelligent Robots
  and Systems}, pp.\  5026--5033. IEEE, 2012.

\bibitem[Toussaint \& Storkey(2006)Toussaint and
  Storkey]{toussaint2006probabilistic}
Toussaint, M. and Storkey, A.
\newblock Probabilistic inference for solving discrete and continuous state
  markov decision processes.
\newblock In \emph{Proceedings of the 23rd international conference on Machine
  learning}, pp.\  945--952. ACM, 2006.

\bibitem[Xie et~al.(2013)Xie, Todorovic, and Zhu]{Xie_2013_ICCV}
Xie, D., Todorovic, S., and Zhu, S.-C.
\newblock Inferring "dark matter" and "dark energy" from videos.
\newblock In \emph{The IEEE International Conference on Computer Vision
  (ICCV)}, December 2013.

\bibitem[Xu et~al.(2018)Xu, Liu, Zhao, and Peng]{xu2018learning}
Xu, T., Liu, Q., Zhao, L., and Peng, J.
\newblock Learning to explore with meta-policy gradient.
\newblock \emph{arXiv preprint arXiv:1803.05044}, 2018.

\bibitem[Zhu et~al.(2019)Zhu, Gupta, Rajeswaran, Levine, and
  Kumar]{zhu2019dexterous}
Zhu, H., Gupta, A., Rajeswaran, A., Levine, S., and Kumar, V.
\newblock Dexterous manipulation with deep reinforcement learning: Efficient,
  general, and low-cost.
\newblock In \emph{2019 International Conference on Robotics and Automation
  (ICRA)}, pp.\  3651--3657. IEEE, 2019.

\bibitem[Ziebart(2010)]{ziebart2010modeling}
Ziebart, B.~D.
\newblock Modeling purposeful adaptive behavior with the principle of maximum
  causal entropy.
\newblock 2010.

\bibitem[Ziebart et~al.(2008)Ziebart, Maas, Bagnell, and
  Dey]{ziebart2008maximum}
Ziebart, B.~D., Maas, A.~L., Bagnell, J.~A., and Dey, A.~K.
\newblock Maximum entropy inverse reinforcement learning.
\newblock In \emph{Aaai}, volume~8, pp.\  1433--1438. Chicago, IL, USA, 2008.

\bibitem[Ziebart et~al.(2009)Ziebart, Ratliff, Gallagher, Mertz, Peterson,
  Bagnell, Hebert, Dey, and Srinivasa]{ziebart2009planning}
Ziebart, B.~D., Ratliff, N., Gallagher, G., Mertz, C., Peterson, K., Bagnell,
  J.~A., Hebert, M., Dey, A.~K., and Srinivasa, S.
\newblock Planning-based prediction for pedestrians.
\newblock In \emph{2009 IEEE/RSJ International Conference on Intelligent Robots
  and Systems}, pp.\  3931--3936. IEEE, 2009.

\end{thebibliography}
\end{document}